%% file: neurips_2025.tex
\documentclass{article}

\usepackage[final]{neurips_2025}
\input{math_commands_icml.tex}

\usepackage{hyperref}
\usepackage[capitalize,noabbrev]{cleveref}
\usepackage{algorithmic}

\theoremstyle{plain}
\newtheorem{theorem}{Theorem}[section]

\newtheorem{lemma}[theorem]{Lemma}
\newtheorem{corollary}[theorem]{Corollary}
\theoremstyle{definition}
\newtheorem{definition}[theorem]{Definition}

\theoremstyle{remark}

\usepackage{subcaption}
\usepackage{float}

\usepackage[utf8]{inputenc} 
\usepackage[T1]{fontenc}    
\usepackage{hyperref}       
\usepackage{url}            
\usepackage{booktabs}       
\usepackage{amsfonts}       
\usepackage{nicefrac}       
\usepackage{microtype}      
\usepackage{xcolor}         

\title{Learning Linear Attention in Polynomial Time}

\author{%
  Morris Yau\\
  MIT CSAIL\\
  \texttt{morrisy@mit.edu} \\
  \And
  Ekin Aky\"urek  \\
  MIT CSAIL\\
  \texttt{akyurek@mit.edu} \\
  \And
  Jiayuan Mao  \\
  MIT CSAIL\\
  \texttt{jiayuanm@mit.edu} \\
  \And
  Joshua B. Tenenbaum  \\
  MIT Brain and Cognitive Sciences\\
  \texttt{jbt@mit.edu} \\
  \And 
  Stefanie Jegelka  \\
  TUM Munich, MCML, MIT CSAIL\\
  \texttt{stefje@mit.edu} \\
  \And
  Jacob Andreas  \\
  MIT CSAIL\\
  \texttt{jda@mit.edu} \\
  }

\begin{document}

\maketitle

\begin{abstract}
  Previous research has explored the expressivity of Transformer models in simulating Boolean circuits or Turing machines. However, the efficient learnability of Transformers from data has remained an open question.  Our study addresses this gap by providing the first polynomial-time learnability results (specifically strong, agnostic PAC learning) for single-layer Transformers with linear attention.  We show that learning the optimal multi head linear attention can be recast as finding the optimal kernel predictor in a suitably defined RKHS.  Moving to generalization, we construct an algorithm that, given a dataset, checks in polynomial time whether the set of best fit multi head linear attention networks on this data all perform an identical computation--a powerful notion for out of distribution generalization.  We empirically validate our theoretical findings on several canonical tasks: learning random linear attention networks, key--value associations, and learning to execute finite automata. Our findings bridge a critical gap between theoretical expressivity and learnability of Transformer models.
\end{abstract}
\input{1-intro}
\input{2-pre}

\input{4-cert}
\input{5-experiments}
\input{7-related-work}

\appendix
\include{8-appendix}

\end{document}

%% file: math_commands_icml.tex
\usepackage{mathtools}
\newcommand{\defeq}{\vcentcolon=}

\usepackage{amsmath,amsfonts,bm}
\usepackage{amsthm}

\theoremstyle{definition}

\def\thmref#1{Theorem~\ref{#1}}
\def\lemref#1{Lemma~\ref{#1}}
\usepackage{thm-restate}

\def\defref#1{Definition~\ref{#1}}

\def\corref#1{Corollary~\ref{#1}}

\usepackage{algorithm}

\def\Figref#1{Figure~\ref{#1}}

\def\secref#1{section~\ref{#1}}
\def\Secref#1{Section~\ref{#1}}

\def\eqref#1{equation~\ref{#1}}

\def\algref#1{algorithm~\ref{#1}}
\def\Algref#1{Algorithm~\ref{#1}}

\def\1{\bm{1}}

\DeclareMathAlphabet{\mathsfit}{\encodingdefault}{\sfdefault}{m}{sl}
\SetMathAlphabet{\mathsfit}{bold}{\encodingdefault}{\sfdefault}{bx}{n}

\newcommand{\E}{\mathbb{E}}

\newcommand{\R}{\mathbb{R}}

\DeclareMathOperator*{\argmax}{arg\,max}
\DeclareMathOperator*{\argmin}{arg\,min}

%% file: 1-intro.tex
\section{Introduction}
Transformers are the dominant neural architecture used in language modeling.
A growing body of work seeks to explain the behavior of trained Transformers and characterize their learnability \citep{perez2019turing,edelman2022inductive,hahn2020theoretical,merrill2023parallelism,merrill2022saturated,merrill2021power,liu2022transformers,feng2023towards,edelman2022,wcm21,zhang2024,trauger2023,chen2024}.  While a large body of work shows that Transformers are \emph{expressive} enough to implement important models of computation, it remains an open question whether these constructions may be efficiently \emph{learned}. Even verifying that a trained model has successfully learned a computational procedure (uniform circuit family) has remained challenging.

Existing work shows positive results on how Transformer-like architectures can express diverse computations, including simulating universal Turing machines~\citep{li2024chain}, evaluating sentences of first-order logic~\citep{barcelo2020logical}, and recognizing various formal languages~\citep{strobl2024formal}.  On the other hand, results on learnability in polynomial time and samples with provable guarantees tend to rely on strong data-generating assumptions, e.g., Gaussian data, etc. \citep{zhang2023,jelassi2022,tian2023,oymak2023,fu2023,tarzanagh2024,deora2023}.  This brings us to our first motivating question.   

\textit{Is there an efficient algorithm in time and samples that learns the optimal parameters of a class of Transformer models for any dataset?}  

In this paper, we establish the strong, agnostic PAC-learnability of linear attention.  Linear attention variants (kernel, gated, flash, etc.) \cite{yang2025, YangWSPK24}, mLSTM in xLSTM \cite{beck2024xlstm}, Retnet \cite{sun2023retentive}, Mamba2 \cite{dao2024transformersssmsgeneralizedmodels}, DeltaNet \cite{schlag2021linear}) have recently matched or outperformed softmax attention in language and vision benchmarks, underscoring the practical value of their theory; \citealp{ahn2024}, \citealp{katharopoulos2020}).  Despite its name, linear attention is not linear and its loss landscape is nonconvex.  We focus our analysis on multi-head linear attention networks, or MHLAs for regression tasks.  An MHLA is parameterized by two matrices $(V_h,Q_h)$ for each of $H$ heads as such $\Theta = \{(V_h, Q_h)\}_{h \in [H]}$. A one layer MHLA computes $Y = \sum_{h \in [H]}V_hZ(Z^TQ_hZ)$.  Here key and query matrices are fused into one, as they multiply one another directly.  

We first show that the computation performed by MHLAs 
can be reformulated as an elementwise product between two larger matrices $\langle W, \mathcal{X}(Z) \rangle$, where $W = \sum_{h \in [H]}\text{flatten}(V_h) \text{flatten}(Q_h)^T $ and $\mathcal{X}(Z)$ is a fixed cubic polynomial function of $Z$. Consequently, optimizing over the class of $H$-head MHLA models is equivalent to optimizing over the class of rank-$H$ matrices $W$.  Furthermore, in the full-rank space of $d^2 \times d^2$ matrices, optimization of $W$ can be performed via linear regression with time polynomial in the inverse target error and size of the dataset. Finally, decomposing an optimal $W$ via SVD recovers an MHLA model with no more than $d^2$ heads that is then guaranteed to compete against the best MHLA parameters---establishing our agnostic learning result (the learned model competes against the best choice of parameters in the hypothesis class). 

Next, achieving zero training and validation loss does not by itself certify that a model has learned a target computation well enough to generalize out of distribution.  Imagine learning arithmetic from input output pairs alone.  Many distinct parameter settings can fit the same data, and fail for larger length inputs.  We therefore ask:

\textit{
Is there a data‑dependent, efficiently checkable condition that forces every empirical‑risk minimiser to realise the same function?
}

For MHLAs the answer is \textbf{yes}.  
Define the second‑moment matrix of the cubic feature map~$\mathcal{X}$ as
\[
\Lambda_D \;=\; \mathbb{E}_{(Z,y)\in D}\!\bigl[\mathcal{X}(Z)\,\mathcal{X}(Z)^{\!\top}\bigr].
\]
If $\Lambda_D$ is full rank—our \emph{certifiable identifiability} criterion—then \emph{all} empirical‑risk minimisers of MHLA coincide on every input.  
The test runs in polynomial time and is unaffected by parameter redundancies such as rescaling $V$ and~$Q$.

Combining this certificate with our expressivity result yields a polynomial‑time procedure that  
(i) learns any circuit family implementable by MHLA whenever the training data satisfy the criterion, and  
(ii) provably recovers, for example, a bounded‑history universal Turing machine from its input–output traces (Appendix~\ref{sec:learn-utm}).  
Once learned, the MHLA simulates any such Turing machine on any input within the prescribed size budget.

In the experimental section, we validate our theoretical findings. In \cref{sec:extraheads}, we train multiple models using stochastic gradient descent on a dataset generated by a single linear attention network's output. Our results demonstrate that multi-head linear attention outperforms both single-layer linear attention and multi-layer linear attention, achieving comparable results to our \Algref{alg:poly}.  In \cref{sec:certificatehelps}, we show that our proposed certificate directly correlates with generalization error even for models trained using stochastic gradient descent. 
In summary:
\begin{itemize}
    \item We provide a polynomial time algorithm that, given any dataset, finds the best fit parameters for multi head linear attention and generalizes with polynomial data, i.e., strong agnostic PAC learning (\cref{sec:polylearn}).   
    
    \item We find an efficiently checkable condition (certifiable identifiability) on the training dataset that certifies every empirical risk minimizer of a MHLA is functionally equivalent, and therefore has the same behavior out of distribution (\cref{sec:cert} see \lemref{lem:certifiable-identifiability}).
    
    \item We study empirically the value of overparameterization with multiple heads vs. multiple layers in \cref{sec:extraheads}.  We verify our certificates empirically on the associative memory task in \cref{sec:certificatehelps}.             
\end{itemize}

%% file: 2-pre.tex
\begin{algorithm}[H]
\small
\caption{MHLA Learning via Regression}
\label{alg:poly}
\begin{algorithmic}[1]
\STATE{\textbf{Input: } Data $D \defeq \{(Z_i,y_i)\}_{i \in [N]}$ for $Z_i  \in \R^{d \times n_i}$ and $y \in \R^d$}
\STATE{ $\{\mathcal{X}_i\}_{i \in [N]} \defeq \text{ExtractFeature}(D)$, generates 
\begin{equation}\label{eq:features}
\mathcal{X}_i \defeq \begin{bmatrix} \langle z_{1:}, z_{1:} \rangle z_{1 n_i} & \langle z_{1:}, z_{2:} \rangle z_{1 n_i} & \cdots & \langle z_{1:}, z_{d:}\rangle z_{dn_i} \\
\langle z_{2:}, z_{1:} \rangle z_{1 n_i} & \langle z_{2:}, z_{2:} \rangle z_{1 n_i} & \cdots & \langle z_{2:}, z_{d:}\rangle z_{dn_i}\\
\vdots & \vdots & \ddots & \vdots \\
\langle z_{d:}, z_{1:} \rangle z_{1 n_i} & \langle z_{d:}, z_{2:} \rangle z_{1 n_i} & \cdots & \langle z_{d:}, z_{d:}\rangle z_{dn_i}\\
\end{bmatrix} ~ .
\end{equation}}
\STATE{Create dataset $\{X_{i,a}\}_{i \in [N], a \in [d]}$. Let $X_{i,a} \in \R^{d^2 \times d^2}$ be a matrix that is comprised of $\mathcal{X}_{i}$ in the $a'th$ block of $d$ rows and $0$ everywhere else:}
\STATE{ \begin{equation}X_{i,a} = \begin{bmatrix} 0 & \hdots & \mathcal{X}_{i}^T & \hdots & 0\end{bmatrix}^T\end{equation}}
\STATE{Let $\hat{W} \in \R^{d^2 \times d^2}$ be regressor:
\begin{equation} \hat{W} \defeq \argmin_{W \in \R^{d^2 \times d^2}}\sum_{i \in [N]} \sum_{a \in [d]} \left(\langle W, X_{i,a}\rangle - 
y_{i,a}\right)^2\end{equation} where $y_{i,a}$ is the $a$'th coordinate of $y_i$.}
\STATE{Take the SVD of $\hat{W} =  AB^T = \sum_{i \in [\hat{H}]} A_i B_i^T$ where $\hat{H}$ is the rank of $\hat{W}$.}
\STATE{$V_h = \text{Fold}(A_h)$ and $Q_h = \text{Fold}(B_h)$ where $\text{Fold}: \R^{d^2} \rightarrow \R^{d \times d}$ takes a vector $p \defeq [p_{ij} \text{ for } i \in [d] \text{ and } j \in [d]]$ and reshapes into a matrix $P \in \R^{d \times d}$ such that $P_{ij} = p_{ij}$.}
\STATE{\textbf{Return: } $\{V_h,Q_h\}_{h \in [\hat{H}]}$}
\end{algorithmic}
\end{algorithm}

\section{Technical Overview}

We start with basic definitions of a multi-head linear attention (MHLA) module, an attention module without the softmax activation.  

\begin{definition} [Multi-Head Linear Attention]
Let $Z \in \R^{d \times n}$ be a matrix of input data.  Let $\Theta = \{(V_h,Q_h)\}_{h \in [H]}$ be a set of parameters where each $V_h, Q_h \in \R^{d \times d}$ denotes value and key-query matrices for all heads $h \in [H]$.  We say $\Theta \in \Omega_{H}$ where $\Omega_H$ is the space of sets of $H$ ordered tuples of $d \times d$ matrices.  We define \emph{multi-head linear attention (MHLA)} to be the function $\text{MHLA}_{\Theta}: \mathbb{R}^{d \times n} \rightarrow \mathbb{R}^{d \times n}$,  
\begin{equation}\label{eq:LA}
\hat{Y} = \text{MHLA}_{\Theta}(Z) = \sum\nolimits_{h \in [H]}V_hZ(Z^TQ_hZ) ~ , 
\end{equation}
where $\hat{Y} \in \R^{d \times n}$ is the output of the one layer linear attention.
We will primarily be interested in the rightmost column vector output by $\text{MHLA}_{\Theta}$ (e.g., as in auto-regressive language models), which is:
\begin{equation}
\hat{y} = \text{MHLA}_{\Theta}(Z) = \sum\nolimits_{h \in [H]}V_hZ(Z^TQ_hZ[:,n]) ~ , 
\end{equation}
where $Z[:,n]$ is the $n$'th column of $Z$.
\end{definition}

\subsection{Polynomial-time learnability}\label{sec:polylearn} Our main result is that $\text{MHLA}$ is learnable in polynomial time.  Colloquially, \Algref{alg:poly} returns an $\text{MHLA}$ that attains the global minimum of the training loss and requires as few as $\text{poly}(d,\epsilon^{-1},\log(\delta^{-1}))$ samples to achieve $\epsilon$ generalization error with probability $1-\delta$.  Our algorithmic guarantees do not require the data to be ``realizable'' (that is, the data need not be generated by an underlying $\text{MHLA}$). 

\begin{restatable}[Learnability of Linear Attention]{theorem}{mainTheorem}
\label{thm:learning}
Let $D$ be a dataset $D = \{Z_i, y_i\}_{i \in [N]}$ drawn i.i.d.\ from a distribution $\mathcal{D}$ where each $Z_i \in \R^{d \times n_i}$, $y_i \in \R^d$. The embedding dimension $d$ is fixed across the dataset, whereas $n_i$ can be different for each datapoint.  Let $n_{max} = \max_{i \in [N]} n_i$ be the maximum sequence length, and let $\Omega_H$ be the space of $H$ pairs of value and key-query matrices $\{(V_h,Q_h)\}_{h \in [H]}$ for any $H \in [1,\infty)$.  Then there is an algorithm (\Algref{alg:poly}) that runs in time $O(Nd^4 n_{max} \epsilon^{-1})$ and that, given input--output pairs $\{(Z_i,y_i)\}_{i \in [N]}$, returns $\hat{\Theta} = \{(\hat{V}_h,\hat{Q}_h)\}_{h \in [\hat{H}]} \in \Omega_{\hat{H}}$ for $\hat{H} \leq d^2$ such that with probability $1 - \delta$,
\begin{multline}
\E_{(Z,y) \in \mathcal{D}}\left[\|\text{MHLA}_{\hat{\Theta}}(Z) - y\|^2\right]\\ - \min\nolimits_{\Theta \in \Omega_H}\E_{(Z,y) \in \mathcal{D}}\left[\|\text{MHLA}_{\Theta}(Z) - y\|^2\right]\leq \epsilon  
\end{multline}
with sample complexity $N = O\left(\frac{1}{\epsilon}\left(d^4 + \log(\delta^{-1})\right)\right)$.
\end{restatable}

Below we describe the high-level ideas behind the algorithm; a formal proof is given in \cref{main-proof}.
Note that if we are purely concerned with guaranteeing that we can find a global minimum of the training loss, we may remove the i.i.d.\ assumption: \Algref{alg:poly} is always within error $\epsilon$ of the optimal training loss.  This is also detailed in \cref{main-proof}.
Specific issues related to generalization over autoregressive sequences rather than i.i.d.\ data 
are handled in the UTM learning result with a standard union bound on the sample complexity; see \Secref{sec:learn-UTM}.      

The main idea behind \Algref{alg:poly} is to construct a feature mapping $\mathcal{X}: \R^{d \times n} \rightarrow \R^{d \times d^2}$ from the data covariates $Z$ with entries $z_{ij}$ for the entry in the $i$'th row and $j$'th column and rows $z_{1:}, z_{2:}, ..., z_{d:} \in \R^{n}$ to a feature space of dimension $d \times d^2$.  The map $\mathcal{X}(Z)$ is defined as:  
\begin{multline}\label{eq:features}
\mathcal{X}(Z) \defeq \\ \begin{bmatrix} \langle z_{1:}, z_{1:} \rangle z_{1 n} & \langle z_{1:}, z_{2:} \rangle z_{1 n} & \cdots  & \langle z_{1:}, z_{d:}\rangle z_{dn} \\
\langle z_{2:}, z_{1:} \rangle z_{1 n} & \langle z_{2:}, z_{2:} \rangle z_{1 n} & \cdots & \langle z_{2:}, z_{d:}\rangle z_{dn}\\
\vdots & \vdots & \ddots & \vdots  \\
\langle z_{d:}, z_{1:} \rangle z_{1 n} & \langle z_{d:}, z_{2:} \rangle z_{1 n} & \cdots &  \langle z_{d:}, z_{d:}\rangle z_{dn}\\
\end{bmatrix} ~ .
\end{multline}
Here, we index the rows of $\mathcal{X}(Z)$ by $j \in [d]$ and the columns by all tuples $(k,\ell) \in [d]^2$ such that $\mathcal{X}(Z)_{j,(k,\ell)} = \langle z_{j:}, z_{k:}\rangle z_{\ell n}$.  At a high level, \Algref{alg:poly} is a kernel method defined by the feature mapping $\mathcal{X}$.  The learned kernel predictor (a regressor) can be mapped back onto a set of parameters $\{\hat{V}_h, \hat{Q}_h\}_{h \in \hat{H}}$ for an $\text{MHLA}$ with no more than $d^2$ heads via SVD. Hence, the relaxation translates into more heads. Interestingly, in our experiments in Section~\ref{sec:extraheads}, $d^2$ heads also benefit learning with SGD.

\paragraph{Proof Idea: }  Much of the notation in this section is defined in \Algref{alg:poly}.  First we write down the loss, and observe that a one-layer attention network is a quadratic polynomial in $\{V_h,Q_h\}_{h \in [H]}$ with input features $X_{i,a}$:
\begin{equation} \label{eq:fundamental}
\mathcal{L}_{\Theta}(\{(Z_i, y_i)\}_{i \in [N]}) = \frac{1}{N}\sum_{i \in [N]}\sum_{a \in [d]} (\left\langle \mathcal{T}_{\Theta}, X_{i,a}\right\rangle - y_{i,a})^2
\end{equation}
with
\begin{multline}\nonumber
\mathcal{T}_{\Theta} \defeq \sum_{h \in [H]}\text{flatten}(V_h) \text{flatten}(Q_h)^T\\ = \sum_{h \in [H]}\begin{bmatrix} V_{h,00} Q_{h,00} & V_{h,00} Q_{h,01} & \hdots & V_{h,00} Q_{h,dd} \\
V_{h,01} Q_{h,00} & V_{h,01} Q_{h,01} & \hdots & V_{h,01} Q_{h,dd}\\
\vdots & \vdots & \vdots \\
V_{h,dd} Q_{h,00} & V_{h,dd} Q_{h,01} & \hdots & V_{h,dd} Q_{h,dd} 
\end{bmatrix} 
\end{multline}

Now we relax this objective by replacing $\mathcal{T}_{\Theta}$ with an unconstrained matrix $W \in \R^{d^2 \times d^2}$.  
While $\mathcal{T}_{\Theta}$ is a rank-$H$ matrix, we allow $W$ to be a general matrix, so this relaxation is guaranteed to have a smaller loss. Furthermore, the loss can be optimized via ordinary least squares. Finally, if we apply SVD to $W$ we obtain a set of $d^2$ left and right singular vectors scaled by the square root the magnitude of the singular value.  Here the scaled left singular vectors correspond to $\hat{V}_h$ and the scaled right singular vectors correspond to $\hat{Q}_h$ for $h \in [\hat{H}]$.  Since the rank of $W$ is no greater than $d^2$ the resulting \text{MHLA} satisfies $\hat{H} \leq d^2$.  The sample complexity follows from classical results in VC theory \citep{kearns94}.  For a full proof see \cref{main-proof}.

\subsection{Identifiability} 

A direct implication of our algorithmic result is the construction of an efficiently checkable condition on the data that guarantees every empirical risk minimizer in a family of $\text{MHLA}$s computes the same function.  Let $\Lambda_D$ be the second moment of a specific mapping $\mathcal{H}(Z)$  of the data, defined in \lemref{lem:certifiable-identifiability}.  
\begin{equation}
    \Lambda_D = \E[\mathcal{H}(Z) \,\mathcal{H}(Z)^T ] = \frac{1}{N}  \sum_{Z \in D}[\mathcal{H}(Z) \, \mathcal{H}(Z)^T].
\end{equation}   
Then if $\Lambda_D$ is full rank or equivalently its minimum eigenvalue is greater than zero, then it is guaranteed that $\text{MHLA}$ is \textit{identifiable with respect to the data}.  
\begin{lemma}[Certificate of Identifiability---Informal]
\label{lem:cert-simple}
Let dataset $D = \{(Z_i, y_i)\}_{i \in [N]}$ be realizable (see \defref{def:realizability}) by an $H$-head \text{MHLA} for any $H \geq 1$.  Let $\mathcal{H}$ be the uniform family of polynomials $\mathcal{H}_{n}: \R^{d \times n} \rightarrow \R^{\psi}$ for $\psi \defeq {d \choose 2}d + d^2$ defined as in \Algref{alg:features-unique}.  For convenience we drop the subscript of $n$ and write $\mathcal{H}(Z)$ to mean $\mathcal{H}_n(Z)$ for $Z \in \R^{d \times n}$.
Finally, define $\Lambda_D \in \R^{\psi \times \psi}$ to be the second moment of the data features:
\begin{equation}
\Lambda_D \defeq \E_{D}\left[\mathcal{H}(Z)\mathcal{H}(Z)^T \right] ~ .
\end{equation}
Then if the eigenvalue $\lambda_{\min}\left( \Lambda_D \right) > 0$, we say that $\text{MHLA}_{\Theta}$ is certifiably identifiable with respect to $D$.  That is, for every pair of empirical risk minimizers $\Theta,\Theta' \in \Omega_H$ 
\begin{equation}
\text{MHLA}_{\Theta} = \text{MHLA}_{\Theta'} 
\end{equation}
i.e., the two models have the same outputs on all inputs. 
\end{lemma}

\begin{corollary}\label{cor:identifiability}
There is a polynomial $p: \Omega_H \rightarrow \R^{\psi}$ such that for any pair of parameters $\Theta,\Theta' \in \Omega_H$ we have $\text{MHLA}_{\Theta} = \text{MHLA}_{\Theta'}$ if and only if  $p(\Theta) = p(\Theta')$.  
\end{corollary}

The polynomial $p$ defines the equivalence class of parameters that compute the same function. For a formal statement of \cref{lem:cert-simple} see \lemref{lem:certifiable-identifiability}.  For handling of errors for approximate empirical risk minimization see \lemref{lem:error-identifiability}.  Moreover, the certificate given by \cref{alg:features-unique} is not the only choice of feature mapping $\mathcal{H}$ that would certify identifiability; \lemref{lem:general-certificate} gives a general certificate for identifiability. 
One way to interpret \cref{cor:identifiability} is that
two $\text{MHLA}$ models parameterized by $\Theta$ and $\Theta'$ compute the same function if and only if they are the same linear function in a specific feature space (akin to matching coefficients in polynomial regression), which in turn is true if 
%
%
$p(\Theta) = p(\Theta')$ for the polynomial $p$ given in \cref{cor:identifiability-formal}. Comparing distance between the coefficients in the range of $p$ is essentially the only meaningful metric of distance that is agnostic to the choice of dataset.  

Finally, we answer a few natural questions related to identifiability which we briefly summarize here.  Firstly, perfectly noisy input data is identifiable under weak assumptions on the moments of the noise (see \lemref{lem:independent}). Secondly, the model class of \text{MHLA} with at least $d^2$ heads is certifiably identifiable from the second moment condition alone, and does not require realizability of the data (see \lemref{lem:arbitrary-heads}).
Finally, we empirically verify the min eigenvalue of $\Lambda_D$ predicts the generalization behavior of SGD for MHLA for the problem of learning key--value memories (see \Figref{fig:assoc-memory-certificate}).

%% file: 4-cert.tex
\section{Application to learning Universal Turing Machines.} \label{sec:utm} In Appendix \ref{sec:realizableutm}, we demonstrate that MHLAs can (autoregressively) express universal Turing machines with polynomially bounded computation histories. In this context, our identifiability results imply that, given a certifiably identifiable dataset of Turing machines and their computation histories on input words, empirical risk minimization and in particular \Algref{alg:poly} will learn the universal Turing machine in a strong sense (\lemref{lem:learning-utm} for learning, \lemref{lem:learnUTMcert} with identifiability). That means at test time the learned MHLA will simulate any Turing Machine on any input word up to a given size for a bounded number of steps. For more detail see \ref{sec:learn-utm}

\begin{restatable}[Learning UTM from Certifiably Identifiable Data]{lemma}{learnUTMcert}
\label{lem:learnUTMcert}  Let $D = \{(Z_i,y_i)\}_{i \in [N]}$ be a dataset satisfying $y_i = \text{MHLA}_{\Theta}$ for $\Theta \in \Omega_H$ being the expressibility parameters of \lemref{lem:utm-expressibility} for the set of TM's/words $(M,x) \in \Delta(\hat{\mathcal{Q}}, \hat{\Sigma}, \hat{n}, \hat{\Phi})$.  If $D$ is certifiably identifiable with $\lambda_{min}(\Lambda_D) > \eta$, then there is a $\text{poly}(d,N,\hat{Q},\hat{\Sigma},\hat{n},\hat{\Phi},\eta^{-1})$ time algorithm that outputs a set of parameters $\hat{\Theta} \in \Omega_{d^2}$ such that for all TM's $M$ and input words $x$ in $\Delta(\hat{\mathcal{Q}}, \hat{\Sigma}, \hat{n}, \hat{\Phi})$, we have 
 \begin{equation}
\text{CH}_{\hat{\Theta}}(M,x)^{c(t)}[:-k_t] = x^t ~ .
 \end{equation}
 The $c(t)$ step of the autoregressive computation history of $\hat{\Theta}$ is equal to the $t$'th step of the computation history of $M$ on $x$.  
\end{restatable}

%% file: 5-experiments.tex
\section{Experiments}
In our experiments, we validate our theoretical predictions in settings where Transformers are trained using stochastic gradient descent (SGD), as follows:  Firstly, \cref{thm:learning} exploits that adding a sufficient number of heads to an MHLA leads to a convex optimization problem after reparameterization.  This suggests that over-parameterizing by adding heads may provide optimization benefits.  We investigate the role of over-parameterization in multi-head and multi-layer linear attention networks.  For random data generated from linear attention networks, we observe that adding more heads achieves faster convergence of training loss than adding more layers.  This suggests that while depth is important for expressiveness, the number of heads is important for optimization (\Figref{fig:val-loss-la}).    

Secondly, we empirically verify the certificate of identifiability provided by \cref{lem:certifiable-identifiability} on datasets for associative memory \citep{bietti2023,cabannes2024} with different choices of embeddings, demonstrating convergence to the equivalence class of the true parameters when $\lambda_{min}(\Lambda_D) > 0$ and converging to spurious solutions when $\lambda_{min}(\Lambda_D) = 0$  (\Figref{fig:assoc-memory-certificate}).

\subsection{Do extra heads help optimization with SGD?}\label{sec:extraheads}

To probe whether more heads facilitate learning in general, we train our convex relaxation and different types of over-parameterized models with SGD on data generated from a single-layer linear attention network. 
For the data, we initialize a single-layer linear attention network with parameters $V \in \mathbb{R}^{1 \times d}$ and $Q \in \mathbb{R}^{d \times d}$, sampled from a Gaussian distribution $\mathcal{N}(0, \frac{I}{\sqrt{d}})$. Input sequences $Z^i \in \mathbb{R}^{T \times d}$ are sampled from $\mathcal{N}(0, \frac{I}{\sqrt{T}})$, where $i=1,\ldots,N$, $T=100$ is the maximum number of time steps, and $N$ is the dataset size. We generate outputs by running the ground-truth network auto-regressively: $y^i_t = VZ^i_{1:t}(Z^i[:,:t]QZ^i[:,t])$, creating our dataset $\mathcal{D} = \{(Z^i, y^i)\}_{i=1}^N$.

In addition to learning with \Algref{alg:poly}, we train three types of models on this data using SGD: (1) multi-head linear attention as in \cref{eq:LA}; (2) multi-layer linear attention with a single head; (3) an ordinary Transformer network \citep{vaswani2017attention} with softmax attention, multi-layer perceptron blocks, and layer normalization.

Figure~\ref{fig:val-loss-lam} illustrates the results. For same experiment with $d=4$ and $N = 2048$ see \cref{fig:val-loss-n-2048-d-4} in the appendix.  Detailed hyperparameters and optimization procedures are described in \cref{app:train-details-la}.

We observe that multi-head attention scales effectively with an increasing number of heads, resulting in improved performance. Notably, for $d=2$ or $4$ input dimensions, using $d^2$ heads 
yields the best performance and is empirically comparable to \Algref{alg:poly}, approaching floating-point error precision. Theoretically, $d^2$ is the maximum rank in the relaxation in  \Algref{alg:poly}.
In contrast, multi-layer attention models show diminishing returns and perform worse than single-layer attention. Interestingly, adding more layers can sometimes degrade performance.
The full transformer model, which incorporates softmax attention, MLP layers and layer normalization, does not significantly outperform the single-layer linear attention model on this task.

These findings suggest that the type of over-parameterization matters significantly in learning linear attention networks. Interestingly, multi-head architectures appear to be particularly effective---aligned with the structure of \Algref{alg:poly}, where the relaxation corresponds to adding more heads.

\begin{figure}[!htp]
    \begin{subfigure}{\textwidth}
        \centering
        \includegraphics[width=0.6\textwidth]{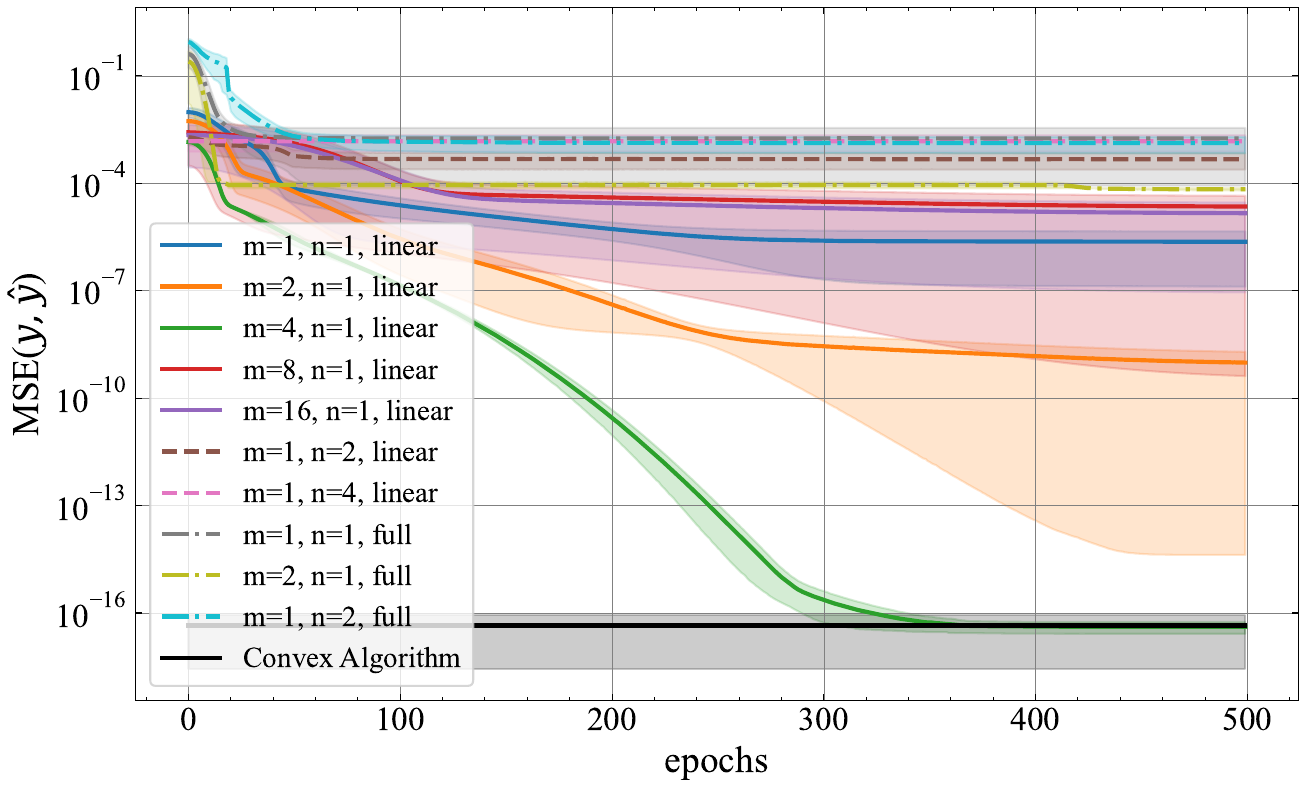}
        \caption{$N=512, d=2$}
        \label{fig:val-loss-n-512-d-2}
    \end{subfigure}

    \begin{subfigure}{\textwidth}
        \centering
        \includegraphics[width=0.6\textwidth]{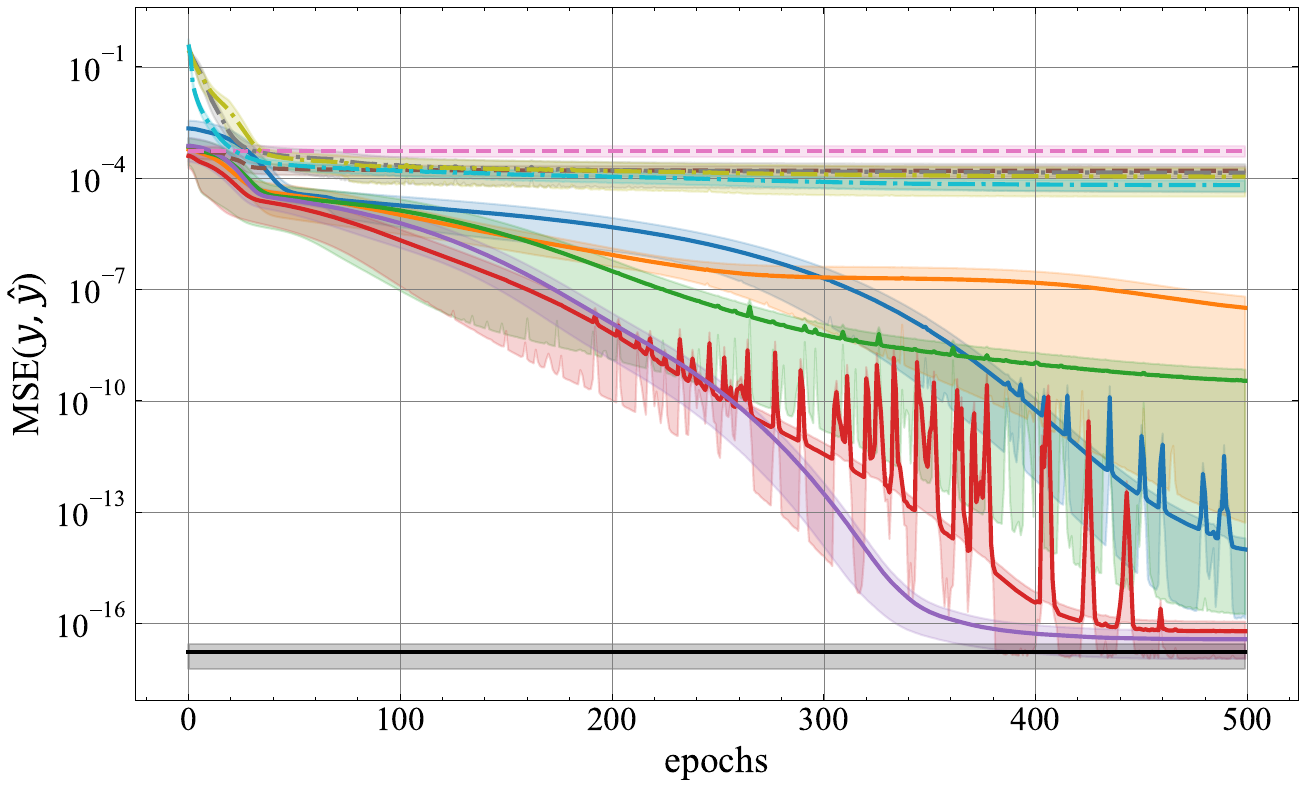}
        \caption{$N=2048, d=4$}
        \label{fig:val-loss-n-2048-d-4}
    \end{subfigure}
    \caption{\textbf{Performance comparison of multi-head, multi-layer linear attention models and the original Transformer model (denoted as \emph{full})}. We trained using SGD on synthetic data generated from a single-layer linear attention model for varying training set sizes ($N$) and input dimensions ($d$), number of heads $m$, and number of layers $n$. 
    Results demonstrate that multi-head architectures converge faster on different input dimensions and match the performance of our \algref{alg:poly} (convex algorithm). Increasing the number of layers or incorporating multilayer perceptrons (MLPs) and layer normalization did not yield consistent improvements. Shading indicates the standard error over three different runs.}
    \label{fig:val-loss-lam}
\end{figure}
%

\subsection{Does certifiable identifiability predict generalization?}\label{sec:certificatehelps}

\begin{figure}[!htp]
    \centering
 \begin{subfigure}[b]{0.4\textwidth}
        \centering
        \includegraphics[width=\textwidth]{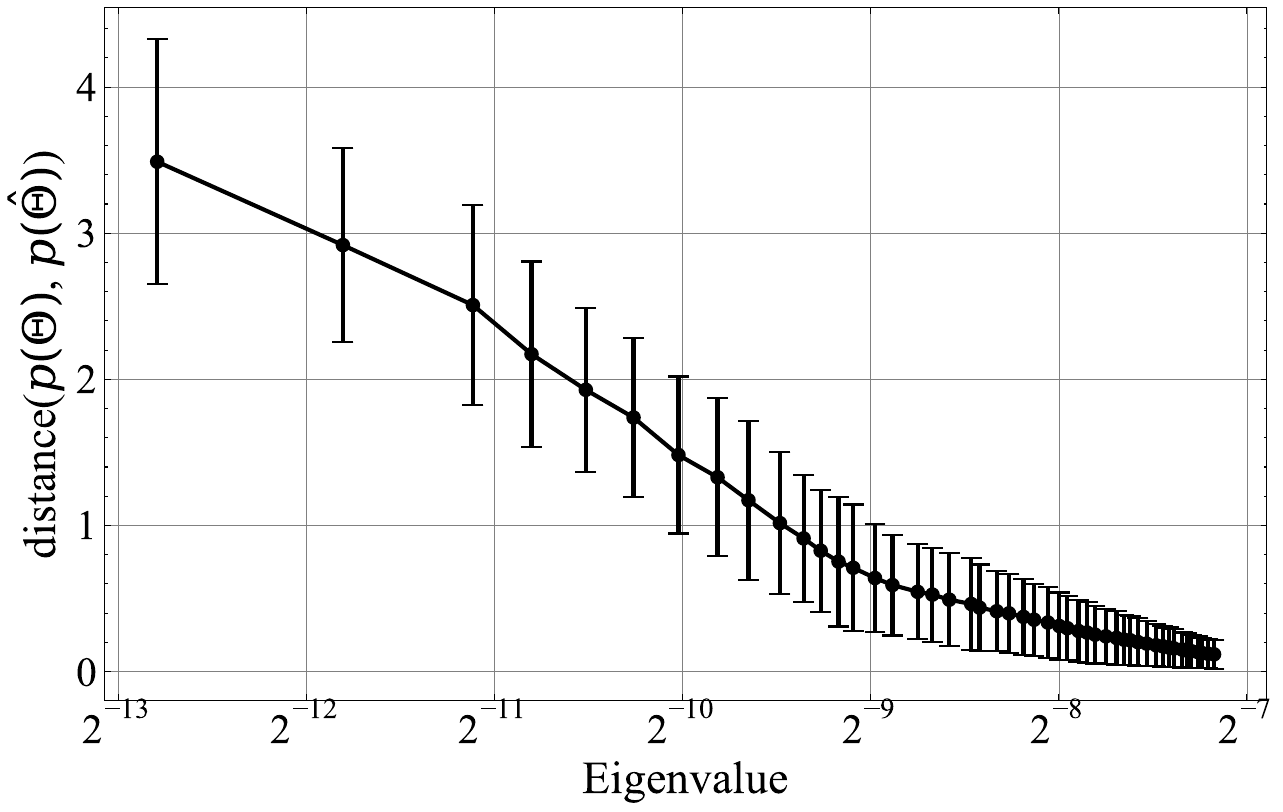}
        \caption{Certificate of identifiability in form of the minimum Eigenvalue $\lambda_{\min}(\Lambda_D)$ vs. Euclidean distance in $p$ feature space of parameters learned by \Algref{alg:poly}.}
        \label{fig:plot1}
        \strut
        \strut
    \end{subfigure}\hspace{15pt}
    \begin{subfigure}[b]{0.5\textwidth}
        \centering
        \includegraphics[width=0.8\textwidth]{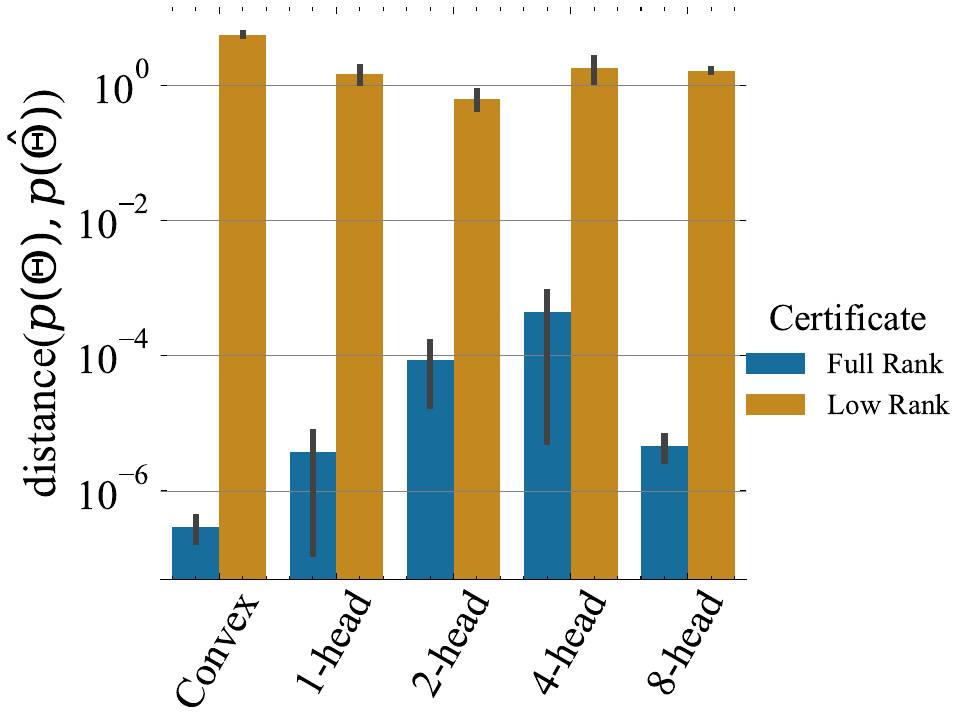}
        \caption{Distance to ground truth parameters in $p$ feature space for certifiably identifiable data (min eigenvalue $= 0.06$) vs.\ nonidentifiable data (min eigenvalue $= 0$).  Here the parameters of MHLA are learned via SGD.  
        Error bars are standard error on three different runs.}
        \label{fig:plot2}
    \end{subfigure}
      \caption{\textbf{Impact of data distribution on the associative lookup task performance:} We generated training data for an associative lookup task \citep{bietti2023,cabannes2024} using mixtures of two distributions: (1) Gaussian key and value vectors, and (2) random unitary key and value vectors. By adjusting the mixture probability, we can manipulate the certificate value (minimum eigenvalue of the data covariance matrix), as unitary key--value vectors give rank-deficient ``certificates''. (a) \Algref{alg:poly}: as the minimum eigenvalue increases, \Algref{alg:poly} converges more closely to the true parameters. (b) SGD: SGD learns parameters that are equivalent to the ground truth parameters in $p$ feature space for certifiably identifiable data, but for unidentifiable data, they are far apart in $p$ feature space and therefore compute different functions.} 
      
    \label{fig:assoc-memory-certificate}
\end{figure}
In \lemref{lem:certifiable-identifiability}, we developed a certificate that provides a sufficient condition for identifiability. To assess the practical relevance of this certificate, we conducted an empirical analysis of convergence in cases where the condition is not satisfied. The results of this analysis are presented in \cref{fig:assoc-memory-certificate}.

\paragraph{Associative Memory} Associative Memory \citep{bietti2023,cabannes2024} is a task of looking up a value in a table with a query.  Via a single head one-layer linear attention model it can be represented with ground truth parameters $\Theta = \{V,Q\}$ where $V,Q \in \R^{2d \times 2d}$:  
\[
\begin{aligned}
V = \begin{bmatrix} 0 & 0\\ 
0 & I_{d \times d}
\end{bmatrix}
\quad
Q = \begin{bmatrix} I_{d \times d} & 0\\ 
0 & 0
\end{bmatrix}.
\end{aligned}
\]
The data $Z$ is drawn as follows: let $k_1,k_2,...,k_d \in \R^{d}$ be random variables corresponding to keys in a lookup table, let $v_1,v_2,...,v_d \in \R^{d}$ be random variables corresponding to values in a lookup table, let $q \in \R^d$ be a random variable corresponding to a query to the lookup table, and $\zeta \sim \mathcal{N}(0,I)$ be random noise, such that $Z$ and the output vector $y$ are defined as:
\begin{gather}
Z = \begin{bmatrix} k_1 & k_2 & \hdots & k_{d} & q \\
v_1 & v_2 & \hdots & v_d & \zeta 
\end{bmatrix}\\
y = \text{MHLA}_{\Theta}(Z) = \begin{bmatrix}
0 \\
\sum_{j \in [d]} \langle q,k_j\rangle v_j
\end{bmatrix}.
\end{gather}
\paragraph{Mixture of distributions:}
We generate two datasets, one that has identifiable $\lambda_{\min}(\Lambda_D) > 0$ and one that is nonidentifiable with $\lambda_{\min}(\Lambda_D) = 0$.  The identifiable dataset is generated with $\{k_j\}_{j \in [d]}$ and $\{v_j\}_{j \in [d]}$ drawn i.i.d $\mathcal{N}(0,I)$. The query $q$ is chosen to be one of the $\{k_j\}_{j \in [d]}$ uniformly at random.  The non-identifiable dataset is drawn such that $\{k_j\}_{j \in [d]}$ forms a random unitary matrix, i.e., $\|k_j\| = 1$ for all $j \in [d]$ and $\langle k_j,k_{j'}\rangle = 0$ for all $j \neq j'$.  Similarly, $\{v_j\}_{j \in [d]}$ is also drawn from a randomly generated unitary matrix.  We draw new random unitary matrices for each datapoint, where  $q$ is again chosen to be one of the $\{k_j\}_{j \in [d]}$ uniformly at random.  We set $d = 4$ dimensions for both datasets, and draw $N = 2^{14}$ samples for each dataset.  We mix the two datasets together with a mixing probability ranging from 95\% unidentifiable to 100\% unidentifiable. In this manner we generate a spread of datasets with different values for $\lambda_{\min}(\Lambda_D)$ that tend to zero.   

\paragraph{Certifiable Identifiability for Algorithm \ref{alg:poly}:}
For each dataset, we run \Algref{alg:poly} which returns $\hat{\Theta}$.  We compare $\hat{\Theta}$ to the ground truth $\Theta$ in $p$ feature space via the 
 distance
\begin{equation}
d(\Theta,\hat{\Theta}) \defeq \| p(\Theta) - p(\hat{\Theta})\|_F.
\end{equation}
Here, $p$ is the polynomial given in \lemref{lem:certifiable-identifiability}.   
Recall from \corref{cor:identifiability-formal} that $p$ defines the equivalence class of parameters that compute the same function, i.e., $\text{MHLA}_{\Theta} = \text{MHLA}_{\hat{\Theta}}$ if and only if $p(\Theta) = p(\hat{\Theta})$.   
On each dataset, we measure the certificate value $\lambda_{\min}(\Lambda_D)$ on the x-axis vs. $d(\Theta,\hat{\Theta})$ on the y-axis.  In \cref{fig:assoc-memory-certificate}a, we see that as the certificate value increases, $d(\Theta,\hat{\Theta})$ decreases, indicating that $\text{MHLA}_{\Theta}$ and $\text{MHLA}_{\hat{\Theta}}$ compute the same function.  

\paragraph{Certifiable Identifiability for MHLA:} Our notion of certifiable identifiability in \lemref{lem:certifiable-identifiability} applies to any empirical risk minimizer.  Therefore, it applies to popular optimizers like SGD and Adam if they achieve the minimum of the loss, which is in our synthetic case equal to zero.  In \Figref{fig:plot2}, we train MHLA models via SGD with $1,2,4,$ and $8$ heads.  For identifiable data with minimum eigenvalue $0.06$, we see that the learned parameters and ground truth parameters are the same in $p$ feature space.  However, for unidentifiable data with minimum eigenvalue $0$, learned parameters and ground truth parameters are far apart in $p$ feature space and therefore compute different functions.




%% file: 7-related-work.tex
\section{Related Work}
\subsection{Formal Expressivity of Transformers}

A large body of work has been trying to tackle the problem of quantifying what algorithmic tasks can a Transformer do, in terms of various kinds of circuit families~\citep{perez2019turing,edelman2022inductive,hahn2020theoretical,merrill2023parallelism,merrill2022saturated,merrill2021power,liu2022transformers,feng2023towards}. In particular, researchers have studied how Transformers can realize specific DSLs~\citep{weiss2021thinking}, logic expressions~\citep{dong2019neural,barcelo2020logical,barcelo2023logical}, Turing machines~\citep{dehghani2018universal,giannou2023looped,perez2021attention}, formal language recognition~\citep{hao2022formal,chiang2023tighter}, as well as automata and universal Turing machines~\citep{liu2022transformers,li2024chain}. However, while these works primarily focus on determining the types of problems whose solutions a Transformer can express, they often overlook the crucial question of how these solutions can be learned from data. Moreover, there is limited discussion on the sufficiency of the dataset itself---whether the data available can identify the underlying ``true'' function or algorithm that we aim to capture.


\subsection{Learning Transformers}
We break down the literature on learning transformers.  First, there is the literature on statistical learnability, where the focus is on the amount of data required to learn without considering whether there is a tractable algorithm for learning \citep{edelman2022,wcm21,zhang2024,trauger2023}.       

Second, there are learnability results for single head transformers for data distributions under a variety of assumptions. In particular, \citet{zhang2023} provide learnability results for in-context linear regression; \citet{jelassi2022} show that data with spatial structure can be learned; the work of \citet{tian2023} analyzes SGD training dynamics for a toy model for data; and \citet{oymak2023} study the prompt attention model.   

Third, the literature on provable guarantees for learning multi head attention is rather sparse.  \citet{fu2023} give learnability results in a regime where attention matrices are fixed and only the projection matrices are trained. \citet{tarzanagh2024} show connections between single layer attention optimization and SVM learning.  Under a good gradient initialization condition, overparameterization condition, and a condition on the scores of optimal tokens the global convergence of gradient descent to a particular SVM problem can be established.  \citet{deora2023} analyze a setting of learning multi head attention with gradient descent under their Assumption 2.  In the words of the authors "these conditions are related to the realizability condition, which guarantees obtaining small training error near initialization", which they instantiate with the separability of the data in an NTK space and a proximity of initialization to realizable parameters.  Interestingly, they find that multi head attention has benign optimization properties.    Finally, \citet{chen2024} study learning for multi head attention for well structured data that is drawn independent Bernoulli or Gaussian.  They provide an extensive discussion of lower bounds for learning multi head attention.

\section{Conclusion and Limitations}

In this work we tackle the fundamental problem of finding an efficient algorithm that provably learns the weights of a linear Transformer.  Our key theoretical ingredient is to consider a model class that's sufficiently "wide" (scaling number of heads), and to find that the loss is convex under this scaling, with generalization guaranteed by the classical VC theory.  This reinforces the empirical observation that scaling model size enables efficient optimization and can still result in successful generalization.  Our theory extends trivially when arbitrary feature maps $\phi(\cdot)$ are applied to keys and queries providing a natural avenue for extending our theory to models that can approximate softmax transformers with custom key-query kernels.  Of course the model class we consider is far simpler than modern LLM's, but we consider our work an important step towards designing algorithms with provable guarantees for training neural sequence models.          
\subsection*{Acknowledgments}
We gratefully acknowledge support from NSF grants IIS-2214177, IIS-2238240, CCF-2112665 and DMS-2134108; from AFOSR grant FA9550-22-1-0249; from ONR MURI grant N00014-22-1-2740; and from ARO grant W911NF-23-1-0034; from the OpenPhilanthropy Foundation; from MIT Quest for Intelligence; from the MIT-IBM Watson AI Lab; from ONR Science of AI; from Simons Center for the Social Brain; and from an Alexander von Humboldt professorship. Any opinions, findings and conclusions or recommendations expressed in this material are those of the authors and do not necessarily reflect the views of our sponsors.

\bibliographystyle{plainnat}
\bibliography{bibliography}

%% file: 8-appendix.tex
\section{Certificate for identifiability of linear attention}\label{sec:cert}
 
We begin by defining identifiability of a model class with respect to a dataset. 

\begin{definition}[Identifiability] \label{def:identifiability}  Let $D = \{(Z_i,y_i)\}_{i \in [N]}$.  Let $\mathcal{U}_{\Theta}$ denote a model class which is a uniform circuit family parameterized by parameters $\Theta \in \Omega$.  Let $\mathcal{L}$ be a loss function and  $\Omega_{\text{ERM}}$ be the set of empirical risk minimizers:
\begin{equation}
\Omega_{\Theta} = \{ \hat{\Theta} \in \Omega \mid \hat{\Theta} = \argmin\nolimits_{\Theta \in \Omega}\mathcal{L}(\mathcal{U}_{\Theta},D) \}.
\end{equation}
We say model class $\mathcal{U}_\Theta$ is \emph{identifiable with respect to the dataset $D$} if  for all $Z \in \R^{d \times n'}$, and for all pairs of empirical risk minimizers $\Theta, \Theta' \in \Omega_{\text{ERM}}$ we have $\mathcal{U}_{\Theta}$ and $\mathcal{U}_{\Theta'}$ compute the same function, i.e., they agree on all inputs (are the same uniform circuit family):  
\begin{equation}
\mathcal{U}_{\Theta}(Z) = \mathcal{U}_{\Theta'}(Z).
\end{equation}
\end{definition}
In establishing conditions for identifiability, it will be useful to refer to another condition relating models to datasets.
\begin{definition}[Realizability]  \label{def:realizability} Let $\Theta \in \Omega_H$ be an MHLA parameterization.  We say a dataset $D = \{(Z_i,y_i)\}_{i \in [N]}$ is \emph{realizable by a parameterization $\Theta$}  if $y_i = \text{MHLA}_{\Theta}(Z_i)$. 
\end{definition}

The definition of realizability can be modified to include independent noise at the expense of adding some terms to our analyses.  See \lemref{lem:error-identifiability} for details.  






Next, we prove that for the model class $\text{MHLA}$ there is an efficiently checkable condition (certificate) of the data $D$ that guarantees the model class is identifiable with respect to $D$.  Our results follow by reinterpreting the results of \thmref{thm:learning} with a focus on data conditions that uniquely determine the optimal regressor.  In this section we denote the mapping from data to feature space to be $\mathcal{H}$ and the mapping from parameters to feature space to be $p$ which are analogous to the $X$ and $\mathcal{T}_{\Theta}$ of \cref{eq:fundamental}.  We instantiate the feature mapping $\mathcal{H}$ and parameter mapping polynomial $p$ as follows.  

\begin{restatable}[Certificate of Identifiability]{lemma}{certifiableIdentifiability}
\label{lem:certifiable-identifiability}
Let dataset $D = \{(Z_i, y_i)\}_{i \in [N]}$ be a realizable dataset.  Let $\mathcal{H} = \{\mathcal{H}_n\}_{n =1}^{\infty}$ be a family of polynomials $\mathcal{H}_{n}: \R^{d \times n} \rightarrow \R^{\psi}$ for $\psi = {d \choose 2}d + d^2$ defined as follows.  We index the entries of $\mathcal{H}$ by taking the Kronecker product between all sets of pairs $\{j,k\}$ (for all $j,k \in [d]$) with with all $\ell \in [d]$. We define $\mathcal{H}(Z)_{\{j,k\} \ell}$ as in \Algref{alg:features-unique} to be 
\begin{equation}
\mathcal{H}(Z)_{\{j,k\}\ell} \defeq \langle z_{j:}, z_{k:}\rangle z_ {\ell n_i}. 
\end{equation}  
Then if $\lambda_{min}\left( \E_{D}\left[\mathcal{H}(Z)\mathcal{H}(Z)^T\right]\right) > 0$, we have that $\text{MHLA}_{\Theta}$ is identifiable with respect to $D$. 
\end{restatable}
Next we construct a mapping $p: \Omega \rightarrow \R^{d \times \psi}$ that partitions the parameter space into  equivalence classes of parameters that compute the same function.  This is akin to matching coefficients in polynomial regression.  This mapping defines a meaningful notion of ``distance'' between different attention parameters by constructing a feature space in which equivalent models have the same representation.  We denote the $a$'th row of $p$ to be $p_a: \Omega \rightarrow \R^{\psi}$ and define it as follows.   
\begin{corollary}\label{cor:identifiability-formal} Let $\{p_a\}_{a \in [d]}$ be a collection of polynomials such that 
$p_a(\Theta): \Omega_H \rightarrow \R^\psi$ is defined as follows.  Each $p_a(\Theta)$ is indexed by pairs $\{j,k\}$ for $j,k \in [d]$ and $\ell \in [d]$ defined to be   
\begin{equation}
p_a(\Theta)_{\{j,k\}\ell} = \sum_{h \in [H]} \left( V_{h,aj} Q_{k\ell} + V_{h,ak} Q_{j\ell} \right) ~ .
\end{equation}  
Let the polynomial $p: \Omega \rightarrow \R^{d \times \psi}$ be $p \defeq (p_1,p_2,...,p_d)$.  Then for any pair of parameters $\Theta,\Theta' \in \Omega_H$ we have $\text{MHLA}_{\Theta} = \text{MHLA}_{\Theta'}$ if and only if  $p(\Theta) = p(\Theta')$.  
\end{corollary}

\begin{algorithm}
\small
\caption{Constructing Features for Certificates of Identifiability}
\label{alg:features-unique}
\begin{algorithmic}[1]
\STATE{\textbf{Input: } Data $D \defeq \{Z_i\}_{i \in [N]}$ for $Z_i \in \R^{d \times n_i}$}
\STATE{\textbf{Output: } feature vectors $\mathcal{H}(Z_i)$ for $i \in [N]$}
\FOR{$Z_i \in D$}
\STATE{Let $z_{1:},z_{2:},...z_{d:}$ be the rows of $Z_i$ and let $z_{ab}$ be the $(a,b)$ entry of $Z_i$}
\FOR{$\text{sets } \{j,k\} \text{ in Distinct Pairs of Indices in } [d]^2$ }
\FOR{$\ell \in [d]$}
    \STATE{$\mathcal{H}(Z_i) = \mathcal{H}(Z_i) \circ \left[\langle z_{j:}, z_{k:} \rangle z_{\ell n_i}\right]$}
\ENDFOR
\ENDFOR
\FOR{$j \in [d]$}
	\FOR{$\ell \in [d]$}
    		\STATE{$\mathcal{H}(Z_i)= \mathcal{H}(Z_i) \circ \left[ \|z_j\|^2z_{\ell n_i}\right]$}
	\ENDFOR
\ENDFOR
\ENDFOR
\STATE{\textbf{Return: } $\{\mathcal{H}(Z_i)\}_{i \in [N]}$}
\end{algorithmic}
\end{algorithm}

We give an overview of a few results building on our certifiable identifiability machinery:

First, data drawn from independent noise is certifiably identifiable.  If the data matrices $\{Z_i\}_{i \in [N]}$ are drawn with each entry being standard normal noise, then $\text{MHLA}_{\Theta}$ for $\Theta \in \Omega_H$  is identifiable with respect to the data. The statement holds beyond standard normals to distributions satisfying weak moment conditions.  The result is stated with population risk instead of empirical risk to simplify the statement.  
\begin{restatable}[Independent input noise yields identifiability]{lemma}{independent}
\label{lem:independent}
 Let $(Z,y) \sim \mathcal{D}$ be a realizable dataset.  Let $Z$ be drawn from a distribution $\mathcal{Z}$ where the $(a,b)$-th entry of $Z$ denoted by $Z_{ab}$ is drawn i.i.d.\ from a distribution $\nu$ over $\R$ for all $a \in [d]$ and $b \in [n]$.  Let the second and fourth moment of $\nu$ be denoted $m_2$ and $m_4$ respectively.  Let $m_2 > 0$ and $m_4 > m_2^2$.  Then $\text{MHLA}_{\Theta}$ for $\Theta \in \Omega_H$ is identifiable with respect to $D$. That is to say, for any population risk minimizers $\Theta, \Theta' \in \Omega_{\text{PRM}}$:
 \begin{equation}
 \text{MHLA}_{\Theta} = \text{MHLA}_{\Theta'}.
 \end{equation}  
\end{restatable}

Second, when specialized to the case of Multi Head Linear Attention $\text{MHLA}_{\Theta}$ with 
 more than $d^2$ heads we can avoid the realizability assumption entirely.  This is because the class of $\text{MHLA}$ with an arbitrary number of heads is linear in the feature space $\mathcal{H}$ given in \lemref{lem:certifiable-identifiability}. 

\begin{restatable}[Identifiability without realizability for MHLA with arbitrarily many heads]{lemma}{arbitraryHeads}
\label{lem:arbitrary-heads}
Let dataset $D = \{(Z_i, y_i)\}_{i \in [N]}$ be any dataset drawn i.i.d from a distribution $\mathcal{D}$.  Let $\mathcal{H}$ be defined as in \lemref{lem:certifiable-identifiability}.  Then if  $\lambda_{\min}\left(\E_D[\mathcal{H}(Z)\mathcal{H}(Z)^T]\right) > 0$ then $\text{MHLA}_{\Theta}$ for $\Theta \in \Omega_{H}$ for any $H \in [d^2,\infty)$ is identifiable with respect to the data $D$.  That is,  
\begin{equation}
\text{MHLA}_{\Theta} = \text{MHLA}_{\Theta'}
\end{equation}
for all pairs of empirical risk minimizers $\Theta, \Theta' \in \Omega_{\text{ERM}}$. 
\end{restatable}

We also add a quantitative version of identifiability with precise treatment of issues related to error.  (For a corresponding statement of realizability with noise see \lemref{lem:error-identifiability-noise}.)  
\begin{restatable}[Identifiability with Error]{lemma}{errorIdentifiability}
\label{lem:error-identifiability}  
Let $\Omega_{\epsilon-\text{ERM}}$ be the set of $\epsilon$-approximate empirical risk minimizers,  
\begin{multline}\nonumber
\Omega_{\epsilon-\text{ERM}} =\\ \left\{ \Theta \in \Omega_H ~\big|~ \E_{(Z_i,y_i) \in D}\left[\left(\text{MHLA}_{\Theta}(Z_i) - y_i\right)^2\right] \leq \epsilon \right\}.
\end{multline}
Then we have for any $\Theta,\Theta' \in \Omega_{\epsilon-\text{ERM}}$ that for all inputs $Z \in \R^{d \times n}$ 
\begin{equation}
\|\text{MHLA}_{\Theta}(Z) - \text{MHLA}_{\Theta'}(Z)\| \leq   \frac{\epsilon}{\lambda_{\min}\left(\Lambda_D\right)} \|Z\|_F^6.
\end{equation}
\end{restatable}

We prove all the above statements in \cref{heads-proof}.

\paragraph{Application to learning Universal Turing Machines.} \label{sec:utm} In Appendix \ref{sec:realizableutm}, we demonstrate that MHLAs can (autoregressively) express universal Turing machines with polynomially bounded computation histories. In this context, our identifiability results imply that, given a certifiably identifiable dataset of Turing machines and their computation histories on input words, empirical risk minimization and in particular \Algref{alg:poly} will learn the universal Turing machine in a strong sense (\lemref{lem:learning-utm} for learning, \lemref{lem:learnUTMcert} with identifiability). That means at test time the learned MHLA will simulate any Turing Machine on any input word up to a given size for a bounded number of steps. For more detail see \ref{sec:learn-utm}

\begin{restatable}[Learning UTM from Certifiably Identifiable Data]{lemma}{learnUTMcert}
\label{lem:learnUTMcert}  Let $D = \{(Z_i,y_i)\}_{i \in [N]}$ be a dataset satisfying $y_i = \text{MHLA}_{\Theta}$ for $\Theta \in \Omega_H$ being the expressibility parameters of \lemref{lem:utm-expressibility} for the set of TM's/words $(M,x) \in \Delta(\hat{\mathcal{Q}}, \hat{\Sigma}, \hat{n}, \hat{\Phi})$.  If $D$ is certifiably identifiable with $\lambda_{min}(\Lambda_D) > \eta$, then there is a $\text{poly}(d,N,\hat{Q},\hat{\Sigma},\hat{n},\hat{\Phi},\eta^{-1})$ time algorithm that outputs a set of parameters $\hat{\Theta} \in \Omega_{d^2}$ such that for all TM's $M$ and input words $x$ in $\Delta(\hat{\mathcal{Q}}, \hat{\Sigma}, \hat{n}, \hat{\Phi})$, we have 
 \begin{equation}
\text{CH}_{\hat{\Theta}}(M,x)^{c(t)}[:-k_t] = x^t ~ .
 \end{equation}
 The $c(t)$ step of the autoregressive computation history of $\hat{\Theta}$ is equal to the $t$'th step of the computation history of $M$ on $x$.  
\end{restatable}

\section{Realizability of Universal Automata in MHLA}\label{sec:realizableutm}

We also include an application of our theory on learnability and identifiability to the problem of learning a universal Turing machine (UTMs) with polynomially bounded computation length.  We prove such a UTM is expressible via MHLA in \lemref{lem:utm-expressibility}, and show that for certifiably identifiable data the learned MHLA generalizes to any TM $M$ and input word $x$ in \lemref{lem:learnUTMcert}.  

\begin{restatable}[UTM Expressibility]{lemma}{UTM}
\label{lem:utm-expressibility}
Let $\Delta(\hat{\mathcal{Q}}, \hat{\Sigma}, \hat{n}, \hat{\Phi})$ be the set of Turing machines $M = \{\delta, \Sigma,\mathcal{Q}, q_{start}, q_{accept}, q_{reject}\}$ and words $x \in \Sigma^*$ with number of states, size of alphabet, size of input, and number of steps in computation history bounded by $\hat{\mathcal{Q}}, \hat{\Sigma}, \hat{n}, \hat{\Phi}$ respectively.  
For any $(M,x) \in \Delta$, let $\{x_t\}_{t \in [\Phi]}$ be the computation history of the UTM on $(M,x)$.  Let the autoregressive computation history (see \defref{def:auto-computation-history}) of $\text{MHLA}_{\Theta}$ on input $(M,x)$ be denoted $\text{CH}_{\Theta}(M,x) = \{Z^1,Z^2,...,Z^{\Phi}\}$.  Then there exists a set of parameters $\Theta \in \Omega_{H}$ for $H = O(\hat{n}\hat{\Phi}\hat{\Sigma})$ and embedding dimension $d = O(\hat{n}\hat{\Phi}\hat{\Sigma} \max(\hat{\Sigma},\hat{\mathcal{Q}}))$, such that for all $(M,x) \in \Delta$, the TM computation history at time step $t$ is equivalent to the autoregressive computation history at time step $c(t)$ where $c(t) \leq O((n+t)t)$ i.e $Z^{c(t)}[:-\text{length}(x^{t}))] = x^{t}$.   Furthermore, this can be achieved with 2 bits of precision.  
\end{restatable}

Our construction bears similarities to \citep{perez2019turing,hahn2020theoretical,merrill2023parallelism,merrill2022saturated,merrill2021power,liu2022transformers,feng2023towards}; the high-level idea is write down every letter in the computation history of $M$ on $x$.  If we use orthogonal vectors to encode every letter, state, and positional embedding we arrive at a natural construction involving a few basic primitives copy, lookup, and if-then-else.  For details see discussion \secref{sec:expressibility} and Proof \ref{proof-expressibility}    

\section{Application to Learning Universal Turing Machines}
\label{sec:learn-utm} We apply our algorithmic and identifiability machinery to 
show that an important computational procedure is representable and learnable as an MHLA: namely, a restricted class of universal Turing machines (UTMs) with bounded computation history.
We must first generalize our previous MHLA definition to enable multi-step computation:
\begin{definition} [Autoregressive MHLA] Let $Z^0$ be an input matrix in dimension $\R^{d \times n}$.  We define the iterative process of  \emph{$\Phi$-step autoregressive MHLA} as follows: starting from $t=0$, let the next token $y^{t+1} \in \R^d$ be:
\begin{equation}
y^{t+1} = \text{MHLA}_{\Theta}(Z^t) ~ ,
\end{equation}
and, for all $t \in [\Phi]$, let $Z^{t+1} \in \R^{d \times (n+1)}$  be the concatenation: 
\begin{equation}
Z^{t+1} =  Z^t   \circ y^t ~  .
\end{equation}
\end{definition}
Next we define the computation history of an autoregressive model analogously to the computation history of a Turing machine.  
\begin{definition} [Autoregressive Computation History] \label{def:auto-computation-history}
We refer to $\text{CH}_{\Theta}(Z) = \{Z^t\}_{t \in [\Phi]}$ as the \emph{computation history} of the $\Phi$-step autoregressive MHLA.  We denote the $t$-th step of the computation history as $\text{CH}_{\Theta}^t(Z) = Z^t$.
\end{definition}

We will often use the notation $Z_t[:-k]$ to denote the last $k \in \mathbb{Z}^+$ tokens of $Z_t$.  Often, $Z$ will be the embeddings corresponding to a word $x$ in a language $\mathcal{L}$, in which case we will use the notation $\text{CH}_{\Theta}(x)$ and $\text{CH}_{\Theta}(Z)$ interchangeably. For pedagogical discussion on how to map embeddings to letters in an alphabet, see \Secref{sec:add-def}         

Although the theory derived in this paper applies to all functions expressible by MHLAs, we are particularly interested in the task of learning {\it universal Turing machines} (UTMs). Let $\Sigma$ be an alphabet. Let $\mathcal{Q}$ be a set of states that includes $\{q_{start}, q_{accept}, q_{reject}\}$ a start, accept, and reject state respectively. Let $\delta: \mathcal{Q} \times \Sigma \rightarrow \mathcal{Q} \times \Sigma \times \{L/R\}$ be a transition function that takes an alphabet and state symbol and maps to a state transition, an output symbol, and a head movement left or right.  Typically there is also a tape alphabet $\Gamma$ for which the input alphabet $\Sigma$ is a subset.      

\begin{definition} [Accept TM] Let $M = \{\delta, \Sigma, \Gamma,\mathcal{Q}, q_{start}, q_{accept}, q_{reject}\}$ be a TM.  Let $x \in \Sigma^*$ be all strings in the alphabet $\Sigma$. Then let $A_{\text{TM}}$ be the language $A_{\text{TM}} = \{(M, x) \mid M \text{ accepts } x\}$.
\end{definition}

The UTM constructed in Turing's 1936 paper recognizes $A_{\text{TM}}$. In practice, we are most often interested in the behavior of TMs that run in polynomial time, and focus below on implementing a universal simulator for this restricted class:


\begin{definition} \label{sec:realizable-utm}(Polynomially Bounded Universal Turing Machine)  In general, a UTM is a recognizer for the language $A_{\text{TM}}$.  That is if $x$ is in $A_{\text{TM}}$, the UTM accepts, else, the UTM rejects or does not halt.  Let    
$A_{\text{TM}} \cap P$ be the language of input pairs $(M,x)$ for TM $M$ and word $x \in \Sigma^*$ such that $M$ decides $x$ in polynomial time. Here, we consider $\text{UTM}$ to be the polynomial time decider for $A_{\text{TM}} \cap P$.
\end{definition}
To define what it means for an autoregressive MHLA to perform the same computation as a TM, our main idea is to construct parameters for MHLA such that it executes the computation history of TM $M$ on input $x$. Let the UTM computation history at step $t$ include the contents $x_0, \ldots, x_{k_t}$ on the tape after $t$ transition steps of the Turing machine $M$, the current state $q_t$, and the current head position $h_t$.  Here $k_t$ is the number of tokens at timestep $t$. Then, there is a single-layer MHLA capable of simulating a UTM:

\UTM*



We include the full proof for the existence of $\Theta$ in the appendix. For simplicity, we adopt a naive embedding scheme that represents different letters in an alphabet as orthogonal unit vectors. This makes it easy to contrive embedding schemes that incorporate arbitrary polynomial-sized circuits which could compute whether $x \in \mathcal{L}(M)$.  Moreover, we adopt positional encodings that are simply orthogonal unit vectors.  Thus, in order to give each of $T$ tokens a unique ID, we would require $O(T)$ dimensional positional embeddings.

This can be combined with the learnability results above to yield a specialized result for UTMs:

\begin{restatable}[Learning a UTM]{lemma}{learningUTM}
\label{lem:learning-utm}
Let $\Theta \in \Omega_H$ in dimension $d$ be the $\text{MHLA}$ parameters in Lemma~\ref{lem:utm-expressibility}.   Let $\{M_i,x_i\}_{i \in [N]}$ be pairs of TM's $M$ and words $x$ of maximum length $n$ drawn i.i.d.\ from a distribution $\mathcal{D}$.  Let $Z_i = \text{Embed}( M_i,x_i)$. For each TM/word pair $(M_i,x_i)$ let $\text{CH}_{\Theta}(Z_i) = \{Z^1_i,Z^2_i,...,Z^{\Phi}_i\}$ be the $\Phi$-step autoregressive computation history of $\text{MHLA}_{\Theta}$ on $Z_i$.  Let $D$ be the dataset $D \defeq \{(\text{CH}_{\Theta}(Z_i)^t, y^{t+1}_{i}\}_{i \in [N], t\in [T]}$ where $y^{t+1}_i = \text{MHLA}_{\Theta}(Z^t_i)$. Then \Algref{alg:poly} applied to input $D$ returns $\hat{\Theta} \in \Omega_{H}$ for $H \leq d^2$ such that with probability $1-\delta$
\begin{equation}
\E_{(Z,y) \in \mathcal{D}}\left[\left(\text{MHLA}_{\hat{\Theta}}(Z) - y\right)^2\right]\leq \epsilon  
\end{equation}
for sample complexity $N = \text{poly}(d,\epsilon^{-1},\log(\delta^{-1}))$.   Then with probability $1-\delta$ over the randomness in the data, the probability over $\mathcal{D}$ that the $\Phi$-step autoregressive computation history $\text{CH}_{\hat{\Theta}}(M,x)$ and $\text{CH}_{\Theta}(M,x)$ differ is upper bounded by 
\begin{equation}
\Pr\nolimits_{(M,x) \sim \mathcal{D}}[\text{CH}_{\hat{\Theta}}(M,x) \neq \text{CH}_{\Theta}(M,x)] \leq O(\epsilon \Phi).
\end{equation}  
\end{restatable}

Finally, if the dataset $D$ is certifiably identifiable, then generalization holds out-of-distribution.  For proof see \cref{proof:learnUTMcert}.  

\learnUTMcert*

\section{Proof of the Main Theorem}
\mainTheorem*
\begin{proof} \label{main-proof}
First we write down the loss:  
\begin{align}
\mathcal{L}_{\Theta}(\{(Z_i, y_i)\}_{i \in [N]}) &\defeq \frac{1}{N}\sum_{i \in [N]}\left\| \sum_{h \in [H]}V_hZ_i (Z_i^T Q_h Z[:,n_i]) - y_i\right\|_F^2\\
&= \frac{1}{N}\sum_{i \in [N]} \sum_{a \in [d]}\left( \sum_{h \in [H]} e_a^TV_hZ_i (Z_i^T Q_h Z[:,n_i]) - y_{i,a}\right)^2
\end{align}

Observe that the one layer attention network is a quadratic polynomial in $\{V_h,Q_h\}_{h \in [H]}$. 

\begin{equation}
= \frac{1}{N}\sum_{i \in [N]}\sum_{a \in [d]} (\left\langle \mathcal{T}_{\Theta}, X_{i,a}\right\rangle - y_{i,a})^2
\end{equation}
Here 
\begin{equation}
\mathcal{T}_{\Theta} \defeq \sum_{h \in [H]}\text{flatten}(V_h) \text{flatten}(Q_h)^T = \sum_{h \in [H]}\begin{bmatrix} V_{h,00} Q_{h,00} & V_{h,00} Q_{h,01} & \hdots & V_{h,00} Q_{h,dd} \\
V_{h,01} Q_{h,00} & V_{h,01} Q_{h,01} & \hdots & V_{h,01} Q_{h,dd}\\
\vdots & \vdots & \vdots \\
V_{h,dd} Q_{h,00} & V_{h,dd} Q_{h,01} & \hdots & V_{h,dd} Q_{h,dd} 
\end{bmatrix} 
\end{equation}

Now we relax the objective where we replace $\mathcal{T}_{\Theta}$ with an unconstrained matrix $W \in \R^{d^2 \times d^2}$.  Another way to put it is that $\mathcal{T}_{\Theta}$ is rank-$H$ but $W$ can be a general matrix. Because the space of general rank matrices is larger, we have written down a relaxation guaranteed to have a smaller loss.  Furthermore the loss can be optimized via ordinary least squares.    

\begin{multline}
\min_{W \in \R^{d^2 \times d^2}}\mathcal{L}_{W}(\{(Z_i, y_i)\}_{i \in [N]}) \defeq \frac{1}{N}\sum_{i \in [N]}\sum_{a \in [d]} (\left\langle W, X_{i,a}\right\rangle - y_{i,a})^2\\ \leq 
\min_{\Theta \in \Omega_{H}} \mathcal{L}_{\Theta}(\{(Z_i, y_i)\}_{i \in [N]}) + \epsilon
\end{multline}

Thus the optimum of the regression with respect to the data achieves optimum of the loss to error $\epsilon$ in time $O(\frac{1}{\epsilon}d^4 N)$.  The sample complexity to achieve error $\epsilon$ is then $O(\frac{1}{\epsilon}(d^4 + \log(\delta^{-1})))$ with probability $1 - \delta$ over the data distribution.  Furthermore, if we take the SVD of $W = \sum_{i \in [\hat{H}]} A_i B_i^T$ where we absorb the singular values into the left and right singular vectors we have for $\hat{\Theta} = \{\text{Fold}(A_h), \text{Fold}(B_h)\}_{i \in [\hat{H}]}$.  Let $\hat{V}_h = \text{Fold}(A_h)$ and $\hat{Q}_h = \text{Fold}(B_h)$
\begin{multline}
\mathcal{L}_{\hat{\Theta}}(\{(Z_i, y_i)\}_{i \in [N]}) \defeq \frac{1}{N}\sum_{i \in [N]}\left\| \sum_{h \in [\hat{H}]}\hat{V}_hZ_i (Z_i^T \hat{Q}_h Z_i[:,n_i]) - y_i\right\|_F^2\\
= \frac{1}{N}\sum_{i \in [N]} \sum_{a \in [d]}\left( \sum_{h \in [\hat{H}]}\hat{V}_hZ_i (Z_i^T \hat{Q}_h Z_i[:,n_i]) - y_{i,a}\right)^2 \leq \epsilon
\end{multline}

as desired.  
\end{proof}

\section{Proofs from Identifiability Section}

First, we start with a general lemma (Lemma~\ref{lem:general-certificate}) which states a sufficient condition for identifiability of any model class that can be written as an inner product of a polynomial of parameters $\Theta$ with a polynomial feature mapping $\mathcal{H}$.  If the data is realizable by the model class and $\Lambda_D = \E_D\left[\mathcal{H}(Z)\mathcal{H}(Z)^T\right]$ is full rank then the model class is identifiable with respect to $D$.  

The following is the certificate of identifiability written in an abstract form involving polynomials to map parameters to feature space and polynomials to map data to feature space.  The proof does not require the model to be an $\text{MHLA}$, but we state it in $\text{MHLA}$ terms for the sake of concreteness. 
\begin{restatable}[General Certificate of Identifiability]{lemma}{generalCertificate} 
\label{lem:general-certificate}    
Let dataset $D = \{(Z_i, y_i)\}_{i \in [N]}$ be a dataset realizable by $\Theta \in \Omega_H$. Let $p \defeq \{p_a\}_{a \in [d]}$ be a collection of polynomials $p_a: \Omega \rightarrow \R^{\psi}$ mapping the parameters $\Theta \in \Omega$ to a feature space of fixed dimension $\psi \in \mathbb{Z}^+$.  Let $\mathcal{H} = \{\mathcal{H}_n\}_{n=1}^{\infty}$ be a uniform family of polynomials such that $\mathcal{H}_n: \R^{d \times n} \rightarrow \R^{\psi}$.  Let  $p$ and $\mathcal{H}$ satisfy 
\begin{equation}
    \text{MHLA}_{\Theta}(Z)[a] = \langle p_a(\Theta), \mathcal{H}_n(Z) \rangle
\end{equation}
for all $Z \in \R^{d \times n}$ for all $n \in [1,\infty)$. 
Then if $\lambda_{\min}\left(\E_D\left[\mathcal{H}(Z)\mathcal{H}(Z)^T\right]\right) > 0$ , we have
\begin{equation}
\text{MHLA}_{\Theta} = \text{MHLA}_{\Theta'}
\end{equation}
for all empirical risk minimizers $\Theta,\Theta' \in \Omega_{\text{ERM}}$.  That is, all empirical risk minimizers compute the same function.  
\end{restatable}   

\begin{proof} \label{general-proof}
 We construct a map $p: \Omega \rightarrow \R^{\psi}$ such that $\text{MHLA}_{\Theta} = \text{MHLA}_{\Theta'}$ if and only if $p(\Theta) = p(\Theta')$.  Then we show that any empirical risk minimizer $\Theta_{\text{ERM}}$ and the ground truth $\bar{\Theta}$ satisfy $p(\Theta_{\text{ERM}}) = p(\bar{\Theta})$.  

In more detail, we construct some polynomials $\{p_a\}_{a \in [d]}$ and family of polynomials $\mathcal{H}$ such that 
\begin{equation}
 \text{MHLA}_{\Theta}(Z)|_a = \langle p_a(\Theta), \mathcal{H}(Z) \rangle  
\end{equation}

We construct a linear model class $\mathcal{R}$ that takes as parameters $v \in \R^{\psi}$ and data $\mathcal{H}(Z) \in \R^{\psi}$. 
such that

\begin{equation}
\mathcal{R}_v(\mathcal{H}(Z)) = \langle v, \mathcal{H}(Z) \rangle
\end{equation}

Let $\Theta_{\text{ERM}}$ be defined as 

\begin{equation}
\Theta_{\text{ERM}} \defeq \{\Theta' \in \Omega | \Theta' = \argmin_{\Theta \in \Omega} \E_{i \in [N]}\left[\mathcal{L}(\text{MHLA}_{\Theta}(Z_i), y_i)\}\right]
\end{equation}

Let $v_{\text{ERM}}$ be defined as  
\begin{equation}
v_{\text{ERM}} \defeq  \{v' \in \R^\psi | v' = \argmin_{v \in \R^\psi} \E_{i \in [N]}\left[\mathcal{L}(\mathcal{R}_{v}(\mathcal{H}(Z_i)), y_i) \right]\}
\end{equation}

Observe that for all $\Theta \in \Theta_{\text{ERM}}$, we have  $p(\Theta)\subseteq v_{\text{ERM}}$.  Here we use the fact that $y$ is realizable by the ground truth $\bar{\Theta}$.  Therefore if we show that $v_{\text{ERM}}$ is unique, i.e comprised of a single element then $p_{\text{ERM}} \defeq \{p(\Theta) | \Theta \in \Theta_{\text{ERM}}\}$ is also unique.  Therefore, $\text{MHLA}_{\Theta}$ is the same function for any $\Theta \in \Theta_{\text{ERM}}$   

To show $v_{\text{ERM}}$ is unique, all we need is that the second moment of the features $\Lambda_D = \E_{D}\left[\mathcal{H}(Z)\mathcal{H}(Z)^T\right]$ is positive definite (the covariance has a minimum eigenvalue bounded away from zero).      
\end{proof}
Next we prove the main certifiable identifiability lemma by instantiating the polynomials $\mathcal{H}$ and $p$ from \lemref{lem:general-certificate}.  
\certifiableIdentifiability*

\begin{proof} \label{certifiable-proof}
First we construct a polynomial $p: \Omega \rightarrow \R^{\psi}$ and $\mathcal{H}: \R^{d \times n} \rightarrow \R^{\psi}$ for $\psi = {d \choose 2}d + d^2$ such that 
\begin{equation}
 \text{MHLA}_{\Theta}(Z)[a] = \langle p_a(\Theta), \mathcal{H}(Z) \rangle  
\end{equation}

We begin by rewriting $\text{MHLA}_{\Theta}(Z)[a]$.  We index the first ${d \choose 2}d$ entries of $p_a(\Theta)$ by all pairs $\{j,k\}$ for $j,k \in [d]$ and all $\ell \in [d]$.  

\begin{equation}
    p_a(\Theta)_{\{j,k\},\{\ell\}} \defeq \sum_{h \in [H]}\left(V_{h,aj}Q_{h,k\ell} + V_{h,ak}Q_{h,j\ell}\right)  
\end{equation}
We define the entries of $p_a(\Theta)$ from $[{d \choose 2}d, {d \choose 2}d + d^2]$ as follows.   
\begin{equation}
    p_a(\Theta)_{\{j^2\}\{\ell\}} \defeq \sum_{h \in [H]}V_{h,aj}Q_{h,j\ell}  
\end{equation}

Similarly, we define $\mathcal{H}(Z)$ be be the following  ${d \choose 2}d + d^2$ features.  $\mathcal{H}(Z)_{\{j,k\} \{\ell\}}$ and $\mathcal{H}(Z)_{\{\ell\}}$.  
\begin{equation}
\mathcal{H}(Z)_{\{j,k\} \{\ell\}} \defeq \langle z_{j:}, z_{k:}\rangle z_ {\ell n} 
\end{equation}
and 
\begin{equation}
\mathcal{H}(Z)_{\{j^2\}\{\ell\}} \defeq \|z_{j:}\|^2 z_{\ell n}
\end{equation}

Thus we rewrite $\text{MHLA}_{\Theta}(Z)[a]$ as 
\begin{multline}
\text{MHLA}_{\Theta}(Z)[a] = \sum_{\{j,k\} \in \mathcal{S}^d_2}  \sum_{\ell \in [d]}p_a(\Theta)_{\{j,k\},\{\ell\}} \mathcal{H}(Z)_{\{j,k\} \{\ell\}} + \sum_{j,\ell \in [d]}  p_a(\Theta)_{\{j^2\}\{\ell\}} \mathcal{H}(Z)_{\{j^2\}\{\ell\}}\\
= \langle p_{a}(\Theta), \mathcal{H}(Z) \rangle 
\end{multline}
Here we introduce the notation $\mathcal{S}_2^d$ to denote the set of all pairs $\{j,k\}$ for $j,k \in [d]$.  We have constructed a polynomial $p_a(\Theta)$ such that for any $\Theta, \Theta' \in \Omega$ in the same equivalence class $p_a(\Theta) = p_a(\Theta')$, we have  $\text{MHLA}_{\Theta} = \text{MHLA}_{\Theta'}$. Furthermore, if there exists $b \in [n]$ such that $\lambda_{min}\left( \E_D\left[\mathcal{H}(Z)\mathcal{H}(Z)^T\right]\right) > 0$ then OLS returns a unique solution for $p_a(\Theta)$.  Since the data is realizable, we conclude $p_a(\Theta) = p_a(\bar{\Theta})$ for all $\Theta \in \Omega_{\text{ERM}}$.     
\end{proof}

Next we present the proof that realizability is not necessary to identify the function learned by MHLA with more than $d^2$ heads.   
\arbitraryHeads*

\begin{proof} \label{heads-proof}
We know from [lemma main algorithm] there exists a surjective map $p_a(\Theta)$ that takes $\Theta \in \Omega$ into $v \in \R^{\psi}$.  This implies that for all $v \in \R^{\psi}$ there exists a right inverse function $p^{r}(v) = \Theta$ satisfying $p(\Theta) = v$ given by SVD.  Therefore,  $p(\Theta_{\text{ERM}}) \in v_{\text{ERM}}$ i.e optimizing over $v \in \R^{\psi}$ does no better than optimizing over $\Theta \in \Omega$. 
To prove this consider the contrary that there exists $v' \in v_{\text{ERM}}$ and there is no $\Theta \in \Omega$ that achieves the same empirical risk as $v'$.  However, $p^r(v) \in \Omega$ is such a $\Theta$, and we have a contradiction.  The key point is that we avoid the assumption of realizability and replace it with surjectivity of the polynomials $p_a$.
\end{proof}

Finally we prove that data drawn from independent noise is certifiably identifiable. 
 A subtlety in the proof is that we use a somewhat different set of polynomials than \lemref{lem:certifiable-identifiability} as we center and normalize our features, which still satisfies the assumptions of the general certificate \lemref{lem:general-certificate}  
\independent*


\begin{proof} \label{independent-proof}
We give the entries of $\Lambda(Z)$ the following naming convention. 
Let the terms  $\{j,k\}\{\ell\}$ and pairs $\{j',k'\}\{\ell'\}$.  Terms that involve $\{j^2\}\{\ell\}$ and $\{j'^2\}\{\ell'\}$ are referred to as 'singles'.    
\begin{equation}
    \E\left[\mathcal{H}_b(Z)_{\{j,k\} \{\ell\} } \mathcal{H}_b(Z)_{\{j',k'\} \{\ell'\} } \right] = \frac{1}{n}\E\left[ \langle z_{j:}, z_{k:}\rangle \langle z_{j':}, z_{k':}\rangle z_{\ell b} z_{\ell'b}\right]
\end{equation}
We give entries of the following form the name "singles to singles"
\begin{equation}
\E\left[\mathcal{H}_b(Z)_{\{j^2\} \{\ell\} } \mathcal{H}_b(Z)_{\{j'^2\} \{\ell'\} } \right] = \frac{1}{n}\E[(\| z_{j:} \|^2 - nm_2)(\|z_{j':}\|^2 - nm_2) z_{\ell b}^2] 
\end{equation}

For the case of $Z$ drawn with each entry i.i.d $\nu$ we can proceed via case work.

\paragraph{Case 1: Pairs to Pairs, $j \neq k$ and $j' \neq k'$}
\begin{enumerate}
    \item \textbf{Subcase 1: $\{j,k\} \neq \{j',k'\}$ and $\ell = \ell'$:}
    \begin{equation}
        \frac{1}{n}\mathbb{E}[\langle z_{j:}, z_{k:}\rangle \langle z_{j':}, z_{k':}\rangle z_{\ell b} z_{\ell' b} ] = 0
    \end{equation}
    
    \item \textbf{Subcase 2: $\{j,k\} = \{j',k'\}$ and $\ell = \ell'$:}
    \begin{equation}
        \frac{1}{n}\mathbb{E}[\langle z_{j:}, z_{k:}\rangle^2 z_{\ell b}^2] = m_2^3
    \end{equation}
\end{enumerate}

\paragraph{Case 2: Singles to Singles, $j = k$ and $j' = k'$}
\begin{enumerate}
    \item \textbf{Subcase 1: $j \neq j'$ and $\ell = \ell'$:}
    \begin{equation}
        \frac{1}{n}\mathbb{E}\left[ \left( \|z_{j:}\|^2 - nm_2 \right)\left( \|z_{j':}\|^2 - nm_2 \right) z_{\ell b}^2 \right] = 0
    \end{equation}
    
    \item \textbf{Subcase 2: $j = j'$ and $\ell = \ell'$:}
    \begin{equation}
        \frac{1}{n}\mathbb{E}\left[\left(\|z_{j:}\|^2 - nm_2\right)^2 z_{\ell b}^2 \right] = \frac{1}{n}\left( (n^2 - n)m_2^2 + nm_4 - n^2m_2^2 \right) m_2 = (m_4 - m_2^2)m_2
    \end{equation}
\end{enumerate}

\paragraph{Case 3: Singles to Pairs, $j = k$ and $j' \neq k'$}
\begin{enumerate}
    \item \textbf{Subcase 1: $\ell = \ell'$:}
    \begin{equation}
        \frac{1}{n}\mathbb{E}\left[ \left( \|z_{j:}\|^2 - nm_2 \right)\langle z_{j':}, z_{k':} \rangle z_{\ell b}^2 \right] = 0
    \end{equation}
\end{enumerate}

Finally for the feature $\mathcal{H}(Z)_{\ell b} = m_2z_{\ell b}$ we have on the main diagonal $\E[m_2^2z_{\ell b}^2] = m_2^2$ and $0$ everywhere else.   

Therefore we've concluded that $\Lambda(Z)$ is a block diagonal matrix because the $\ell \neq \ell'$ blocks are near zero.  All that remains is to verify that the diagonal blocks are full rank. 
\begin{enumerate}
\item Pairs to Pairs: $m_2^3I$ is full rank with min eigenvalue $m_2^3$
\item Singles to Singles: $(m_4 - m_2^2)m_2I$ is full rank with min eigenvalue $(m_4 - m_2^2)m_2$.     
\end{enumerate}
\end{proof}

Finally we provide a simple error bound for approximate empirical risk minimizers to demonstrate the robustness of the conclusions in \lemref{lem:certifiable-identifiability}.  
\errorIdentifiability*

\begin{proof} \label{proof:error-identifiability}
\begin{multline}
\|\text{MHLA}_{\Theta}(Z) - \text{MHLA}_{\Theta'}(Z)\|^2 = \sum_{a \in [d]}\left(\langle p_a(\Theta) - p_a(\Theta'), \mathcal{H}(Z)\rangle\right)^2  \\ 
\leq  \sum_{a \in [d]} \| p_a(\Theta) - p_a(\Theta')\|^2 \|\mathcal{H}(Z)\|^2\\
\leq  \left(\sum_{a \in [d]} \| p_a(\Theta) - p_a(\Theta')\|^2 \right) \|Z\|_F^6\\
\leq  \frac{\epsilon}{\lambda_{min}\left(\Lambda_D\right)} \|Z\|_F^6\\
\end{multline}
Here the first equality follows from the linearization exhibited in \lemref{lem:general-certificate}.  The first inequality is cauchy schwarz.  In the second inequality we apply a crude upper bound that no more than 6'th degree polynomials that are products of three squares of entries in $Z$ are involved in $\|\mathcal{H}(Z)\|^2$.   

\begin{equation}
\|\mathcal{H}(Z)\|^2 \leq \sum_{a,a',a'' \in [d]\text{, }b,b',b'' \in [n]} Z_{ab}^2 Z_{a'b'}^2 Z_{a''b''}^2 \leq \|Z\|_F^6
\end{equation}
The last inequality comes from the fact that $\Theta,\Theta'$ are $\epsilon$ approximate empirical risk minimizers.  Therefore we know
\begin{equation}
\lambda_{min}(\Lambda_D) \sum_{a \in [d]}\|p_a(\Theta) - p_a(\Theta')\|^2 \leq \sum_{a \in [d]}\left(\langle p_a(\Theta) - p_a(\Theta'), \mathcal{H}(Z)\rangle\right)^2 \leq \epsilon 
\end{equation}
which implies 
\begin{equation}
 \sum_{a \in [d]}\|p_a(\Theta) - p_a(\Theta')\|^2 \leq \frac{\epsilon}{\lambda_{min}(\Lambda_D)} 
\end{equation}
which concludes the proof.  
\end{proof}

\begin{lemma}[Identifiability with Error and Noise in Realizability]
\label{lem:error-identifiability-noise} Let $D = \{(Z_i,y_i)\}_{i \in [N]}$ be a dataset such that $y_i = \text{MHLA}(Z_i) + \zeta_i$ for $\zeta_i$ i.i.d and bounded.  
Let $\Omega_{\epsilon-\text{ERM}}$ be the set of $\epsilon$-approximate empirical risk minimizers.  
\begin{equation} 
\Omega_{\epsilon-\text{ERM}} = \left\{ \Theta \in \Omega_H ~\big|~ \E_{(Z_i,y_i) \in D}\left[\left(\text{MHLA}_{\Theta}(Z_i) - y_i\right)^2\right] \leq \epsilon \right\}.
\end{equation}
Let $\max_{i \in [N]}\|Z_i\|_F \leq B$ . Then we have for any $\Theta,\Theta' \in \Omega_{\epsilon-\text{ERM}}$ that for all inputs $Z \in \R^{d \times n}$ 
\begin{equation}
\|\text{MHLA}_{\Theta}(Z) - \text{MHLA}_{\Theta'}(Z)\| \leq   \frac{\epsilon - \frac{1}{N}\sum_{i \in [N]} \zeta_i^2 + \frac{B^2}{N} \log(\delta^{-1})}{\lambda_{\min}\left(\Lambda_D\right)} \|Z\|_F^6.
\end{equation}
\end{lemma}
\begin{proof}
The proof follows directly from \lemref{lem:error-identifiability} but we incorporate the $\zeta_i$ terms as is standard in analyses of linear regression.  
\end{proof}

\clearpage

\section{Programs Expressible as Fixed Depth Linear Transformer}
\label{sec:expressibility}
In this section we build out examples of programs that can be expressed as fixed depth linear transformers.  Expressibility results can be carried out in a variety of equivalent ways.  The main takeaway, is that the computation history of TM $M$ on word $x$, when written down "step by step" can be captured by next token prediction of linear attention.  This is because the key-query-value naturally implements a table lookup sometimes referred to as "associative memory" or "in context linear regression" in the linear case.

The notion of an Autoregressive MHLA Program is useful for condensing the proofs of expressibility. We write such programs in an object oriented syntax with each token representing an object with multiple attributes.  Attributes can be updated and looked up from other objects using a generalized lookup akin to associative memory.  

\begin{algorithm}
\small
\caption{Autoregressive MHLA Program}
\label{alg:MHLA}
\begin{algorithmic}[1]
\STATE{Instantiate N instances OBJ = $\{\text{obj(i)}\}_{i \in [N]}$ of Class with set of Attributes $\{\text{Attr}_1, \text{Attr}_2, ..., \text{Attr}_k\}$}
\STATE{Each Attribute takes on values in an alphabet $\Sigma_{\text{Attribute}}$}
\FOR{iter $\in$ [T]}
\STATE{Let $\text{obj}[r]$ be the rightmost token}
\STATE{Let $\text{obj}[r+1]$ be a new token initialized with positional embedding $\text{obj}[r+1].\text{pos} = r+1$}
    \FOR{each \{AttrSource, AttrDest\} in \{Pairs of Attributes in Class\}}
        \STATE{\text{\#}AttrKey and AttrValue can be any pair of Attributes (and can be distinct from VarSource/VarDest)}
        \STATE{$\text{LookupDict} = \{\{\text{obj.AttrKey: obj.AttrValue}\} \text{ for obj in OBJ}\}$ }
        \STATE{\text{\#} if multiple objects have same obj.AttrKey then returns sum of obj.AttrValues which we aim to avoid}
            \STATE{Let $\mathcal{B}_Q$ be any function from $\Sigma_{\text{AttrSource}}$ to $\Sigma_{\text{AttrKey}}$}
            \STATE{Let $\mathcal{B}_V$ be any function from $\Sigma_{\text{AttrValue}}$ to $\Sigma_{\text{AttrDest}}$}
            \STATE{Let query = $\mathcal{B}_Q$(obj[r].AttrSource)}
            \IF{query in LookupDict.Keys}
                \STATE{obj[r+1].AttrDest = $\mathcal{B}_V$(LookupDict(query))}
            \ENDIF
    \ENDFOR
\STATE{Append next token $\text{OBJ} = \{\text{obj}[i]\}_{i \in [r]} \cup \{\text{obj}[r+1]\}$ }
\STATE{$r = r+1$}
\ENDFOR
\end{algorithmic}
\end{algorithm}

\begin{lemma} For any program $\mathcal{P}$ written in the form of \algref{alg:MHLA}, there exists corresponding MHLA parameters $\Theta \in \Omega_{H}$ such that $\text{MHLA}_{\Theta}(Z) = \mathcal{P}(Z)$.  
\end{lemma}

\begin{proof}
We set some matrices to implement lookup tables.  For any function of $f: A \rightarrow B$ for sets $A$ and $B$ there is a canonical representation of the input domain as orthogonal unit vector $v_1,v_2,..., v_{|A|} \in \R^{A}$ and output domain as another set of orthogonal unit vectors $u_1,u_2,...,u_{|B|} \in \R^{B}$.  Therefore, there is a matrix $G_f$ that maps input vectors to output vectors satisfying $G_f v_i = u_j$ for $j = f(i)$ for all $i \in [A]$ and $j \in [B]$.    

For functions $f: \Sigma_{\text{AttrSource}} \rightarrow \Sigma_{\text{AttrKey}}$ and $f': \Sigma_{\text{AttrValue}} \rightarrow \Sigma_{\text{AttrDest}}$ we associate matrices $B_Q \in \R^{|\Sigma_{\text{AttrSource}}| \times |\Sigma_{\text{AttrKey}}|}$ and $B_V \in \R^{|\Sigma_{\text{AttrValue}}| \times |\Sigma_{\text{AttrDest}}|}$ respectively.  

Then we form $\{V_h,Q_h\}_{h \in [H]}$ as follows.  Let $V$ be the matrix that is all zeros with $B_V$ in the rows associated with $\Sigma_{\text{AttrSource}}$ and the columns associated with $\Sigma_{\text{AttrKey}}$.  Let $Q$ be the matrix that is all zeros with $B_V$ in the rows associated with $\Sigma_{\text{AttrValue}}$ and the columns associated with $\Sigma_{\text{AttrDest}}$.

In each layer we have multiple heads, each one performs the lookup operation for each pair of attributes in the class. 
\end{proof}
\subsection{Construction of UTM}
Now we proceed with our construction of an Autoregressive MHLA-Program for UTM.  The UTM requires a small number of operations captured by an Autoregressive MHLA-Program.

We define an embedding function that takes as input a TM $M$ and word $x$ such that  
\begin{definition} [Embedding] Let $M$ be a TM over state space $Q$, alphabet $A$, transition function $\delta$.  Then 
\begin{equation}
\text{Embedding}(M) = \begin{bmatrix}
    q_0 & q_1 & \cdots & q_{k} &\# \\
    a_0 & a_0 & \cdots & a_0 & \# \\
    \delta(q_0,a_0) & \delta(q_1,a_0) & \cdots & \delta(q_{k}, a_0)& \#\\
    a_1 & a_1   & \cdots &  a_1&\#\\
     \delta(q_0,a_1) & \delta(q_1,a_1) & \cdots & \delta(q_{k}, a_1) &\#\\
\end{bmatrix}
\end{equation}
Let $p_1,p_2,...,p_{\delta}$ be "positional encodings" that assign unique id's for every letter in the word $x$.   
\begin{equation}
\text{Embedding}(x) = 
\begin{bmatrix}
    p_1 & p_2 & \cdots & p_i & p_{i+1} & \cdots & p_{\delta}&\# \\
    x_1 & x_2 & \cdots & x_i & x_{i+1} & \cdots & x_{\delta}&\# \\
    0 & 0 & \cdots & q & 0 & \cdots & 0 & \#
\end{bmatrix}
\end{equation}
Then we define $\text{Embedding(M,x)}$ to be

\begin{equation}
\text{Embedding(M,x)} = \begin{bmatrix}
\text{Embedding}(M) & 0 \\
0 & \text{Embedding}(x)
\end{bmatrix}
\end{equation}
\end{definition}
Henceforth we will write the construction in the syntax of an Autoregressive MHLA-Program instead of matrices with blocks of zeros and token embeddings to save space.  

\UTM*

The construction is given in the language of Autoregressive MHLA-Programs in \algref{alg:MHLA} which provides the instruction set for writing the next letter in the computation history onto the output tape.  
\begin{proof} \label{proof-expressibility}
\textbf{Proof Idea: } A few elementary operations can be captured by a MHLA-program which can be composed to output the computation history of $M$ on $x$.  We begin by introducing some notation for the "Lookup" operation which we build into copy, move, and if-then which are all the operations required to construct the UTM.    

\textbf{General Lookup: } For each lookup there are three objects that are involved. Let Token$ = \text{obj}[r]$ be the "source" which is always the rightmost token.  An attribute from the source object known as AttrSource is  linearly transformed to form a "query".  Lookup involves a table $T = \{\text{obj[i].AttrKey: obj[i].AttrValue}\}_{i \in [r]}$ which is used to match an AttrKey to look up an AttrValue from an object $\text{obj}[p]$ that we denote the "target".  Note, that if the obj[i] has an AttrKey that is zero, it is the same as not being in the table.  In the pseudocode \algref{alg:MHLA} these zero attributes are denoted as "None".

Given a query, we copy the associated AttrValue from the lookup table $T$ and update AttrDest in an object NextToken$ = \text{obj}[r+1]$ which we denote the "destination".  Multiple lookup operations can be performed in parallel by many heads with each head responsible for a single lookup.  

To output each letter of the computation history, we increase the number of tokens $r$ by a constant $c$.  We refer to the set of contiguous tokens $[0,c], [c, 2c], etc.$  involved in the computation of a single letter as a "block".  Here block[i] = $\{\text{obj}[j]\}_{j \in [ic,(i+1)c]}$.  We construct a different set of heads to act on each token and enforce that the nonzero rows that each block of tokens occupy are disjoint.  Furthermore, within a block, the states of each token occupies a disjoint set of rows except when they are used to construct a table.  Tables are the only case where we want tokens to occupy the same rows.  In this manner the following abstraction can be made. 

At the beginning of each block starting with obj[r], we can lookup attributes from anywhere in OBJ that we want to load into different attributes in obj[r].  Then we can apply any sequence of if-then statements involving the attributes of obj[r] to update the attributes (or create new attributes).   
To run the UTM we need a few simple primitives denoted Lookup and If-Then.  

\textbf{Construction of Primitives: } We write down the construction by constructing a sufficient set of primitives Lookup and If-Then.  We also include Copy which is a special case of Lookup that is used frequently.

\paragraph{Lookup: } When the transforms $B_Q$ and $B_V$ are the identity we denote the lookup operation for table $T$ where we query an attribute $s'$ of $\text{obj}[r]$ to update the attribute $s$ of $\text{obj[r+1]}$ as $\text{obj[r+1].s} = \text{Lookup(T,obj[r].s')}$    
\paragraph{Copy: }  A special case of lookup is copy, where we need to copy attributes from tokens that are at an offset $-k$ for $k \in [r]$.  This can be done by setting $\mathcal{B}_Q$ to permute the positional encoding by $-k$ positions. Then the query matches the key that is the positional encoding of the target object. Let $s,s'$ be target and destination attributes.  We denote the copy operation of the attribute $s'$ of the obj at offset $-k$ from $r$ into the attribute $s$ of the destination object to be $\text{obj[r+1].s} = \text{Copy}(\text{obj[r-k].s'})$. 
\paragraph{If-Then: } We write down an If-Then Program \algref{alg:if-then} and a corresponding Autoregressive MHLA-Program \algref{alg:MHLA-if-then} to implement If-Then.  An If-Then program looks up whether an attribute $x$ is equal to any of attributes $a_1,a_2,...,a_k$ then we set attribute $x'$ to $b_1,b_2,...,b_k$ respectively.  This is achieved by copying the attributes $a_i$ and $b_i$ into dummy attributes $s0$ and $s1$ for all $i$ in $k$ for a series of $k$ consecutive tokens.  This creates a table with key $s0$ and value $s1$.  Then we use attribute $x$ as the query, which looks up the corresponding value $s1$ which we use to update an attribute $x'$.   

\begin{algorithm}
\small
\caption{If-Then Program}
\label{alg:if-then}
\begin{algorithmic}[1]
\STATE{\# If attribute x is equal to any of $a_1,a_2,...,a_k$ then set attribute $x'$ to $b_1,b_2,...,b_k$ respectively}
\IF{Token.x == Token.$a_1$: }
    \STATE{NextToken.x' = Token.$b_1$}
\ENDIF
\IF{Token.x == Token.$a_2$: }
    \STATE{NextToken.x' = Token.$b_2$}
\ENDIF
\STATE{$\hdots$}
\IF{Token.x == Token.$a_k$: }
    \STATE{NextToken.x' = Token.$b_k$}
\ENDIF
\end{algorithmic}
\end{algorithm}

\begin{algorithm}
\small
\caption{MHLA If-Then Program}
\label{alg:MHLA-if-then}
\begin{algorithmic}[1]
\STATE{\# If attribute x is equal to any of $a_1,a_2,...,a_k$ then set attribute $x'$ to $b_1,b_2,...,b_k$ respectively}
\STATE{token[r+1].s0 = token[r].$a_1$}
\STATE{token[r+1].s1 = token[r].$b_1$}
\STATE{NEXT TOKEN $r = r+1$}
\STATE{token[r+1].s0 = token[r].$a_2$}
\STATE{token[r+1].s1 = token[r].$b_2$}
\STATE{$\hdots$}
\STATE{NEXT TOKEN $r = r+1$}
\STATE{token[r+1].s0 = token[r].$a_k$}
\STATE{token[r+1].s1 = token[r].$b_k$}
\STATE{NEXT TOKEN $r = r+1$}
\STATE{Table T = $\{\text{obj[i].s0 : obj[i].s1}\}_{i \in [r,r-k+1]}$}
\STATE{token[r+1].x' = Lookup(T,token[r].x)}
\end{algorithmic}
\end{algorithm}
\end{proof}
\begin{algorithm}
\small
\caption{Simplified Instruction Set MHLA Program for UTM for a single block}
\label{alg:MHLA}
\begin{algorithmic}[1]
\STATE{\# Initialize Lookup Tables for TM M and tape $T_1$}
\STATE{\# $\delta(q,a) = [\text{next-state, next-letter, next-move}]$}
\STATE{M = $\{q: [a_0, \delta(q,a_0), a_1, \delta(q,a_1)]\}_{q \in Q}$ }
\STATE{$T_1$ = $\{\text{token[i].PosEncoding: token[i].Letter}\}_{i \in [r]}$ }
\STATE{\# Begin Loading Information from M and previous tokens on tape}
\STATE{\# First copy letter/state from token -N-1 positions away}
\STATE{\# Attribute s(-1) = \{letter, state\} where state can be equal to None}
\STATE{NextToken.s(-1) = Copy(Token[-N-1].s0)}
\STATE{\# Second copy letter/state from token -N positions away}
\STATE{\# Attribute s0 = \{letter, state\} where state can be equal to None}
\STATE{NextToken.s0 = Copy(Token[-N].s0)}
\STATE{\# Third copy letter/state from token -N+1 positions away}
\STATE{\# Attribute s1 = \{letter, state\} where state can be equal to None}
\STATE{NextToken.s1 = Copy(Token[-N+1].s0)}
\STATE{NEXT TOKEN $r = r+1$}
\STATE{\#Split into three branches to handle left, head, and right positions relative to head}
\STATE{\textbf{RUN BRANCH 1} (Token is Left of Head Position) See \algref{alg:branches}}
\STATE{\textbf{RUN BRANCH 2} (Token is at Head Position)  See \algref{alg:branches}}
\STATE{\textbf{RUN BRANCH 3} (Token is Right of Head Position)  See \algref{alg:branches}}
\end{algorithmic}
\end{algorithm}

\begin{algorithm}
\small
\caption{Branches to handle cases Left of Head, Head, and Right of Head}
\label{alg:branches}
\begin{algorithmic}[1]
\STATE{\#Split into three branches to handle left, head, and right positions relative to head}
\STATE{\textbf{BRANCH 1} (Token is Left of Head Position)} 
\STATE{\# we have loaded a state q into s1 (if left of head) and next we load $[a_0, \delta(q,a_0), a_1, \delta(q,a_1)]$ into s2}
\STATE{NextToken.s2 = Lookup(M,Token.s1.state)}
\STATE{NEXT TOKEN $r = r+3$}
\IF{$\text{Token.s2.letter} == a_0$}
    \STATE{NextToken.s3 = $\delta(q,a_0)$ = [q',w',L/R]}
\ENDIF
\IF{$\text{Token.s2.letter} == a_1$}
    \STATE{NextToken.s3 = $\delta(q,a_1)$ = [q',w',L/R]}
\ENDIF
\STATE{NEXT TOKEN r = r+3}
\IF{Token.s3.move == L}
    \STATE{NextToken.return-letter = Token.s0.letter}
    \STATE{NextToken.return-state = q'}
\ENDIF
\IF{Token.s3.move == L}
    \STATE{NextToken.return-letter = Token.s0.letter}
    \STATE{NextToken.return-state = None}
\ENDIF
\STATE{\textbf{BRANCH 2} (Token is at Head Position)} 
\STATE{\# we have loaded a state q into s0 and next we load $[a_0, \delta(q,a_0), a_1, \delta(q,a_1)]$ into s2}
\STATE{NextToken.s2 = Lookup(M,Token.s0.state)}
\STATE{NEXT TOKEN r = r+3}
\IF{$\text{Token.s2.letter} == a_0$}
    \STATE{NextToken.s3 = $\delta(q,a_0)$ = [q',w',L/R]}
\ENDIF
\IF{$\text{Token.s2.letter} == a_1$}
    \STATE{NextToken.s3 = $\delta(q,a_1)$ = [q',w',L/R]}
\ENDIF
\STATE{NEXT TOKEN r = r+3}
\IF{Token.s3.next-letter is not None}
    \STATE{NextToken.return-letter = Token.s3.next-letter}
    \STATE{NextToken.return-state = None}
\ENDIF
\STATE{\textbf{BRANCH 3} (Token is Right of Head Position)} 
\STATE{\# we have loaded a state q into s(-1) and next we load $[a_0, \delta(q,a_0), a_1, \delta(q,a_1)]$ into s2}
\STATE{NextToken.s2 = Lookup(M,Token.s(-1).state)}
\STATE{NEXT TOKEN r = r+3}
\IF{$\text{Token.s2.letter} == a_0$}
    \STATE{NextToken.s3 = $\delta(q,a_0)$ = [q',w',L/R]}
\ENDIF
\IF{$\text{Token.s2.letter} == a_1$}
    \STATE{NextToken.s3 = $\delta(q,a_1)$ = [q',w',L/R]}
\ENDIF
\STATE{NEXT TOKEN r = r+3}
\IF{Token.s3.move == L}
    \STATE{NextToken.return-letter = Token.s0.letter}
    \STATE{NextToken.return-state = None}
\ENDIF
\IF{Token.s3.move == R}
    \STATE{NextToken.return-letter = Token.s0.letter}
    \STATE{NextToken.return-state = Token.s3.next-state}
\ENDIF
\end{algorithmic}
\end{algorithm}

\subsection{Proofs For Learning UTM}
\label{sec:learn-UTM}

\learningUTM*

\begin{corollary} \label{cor:learning-utm}
In particular, for sample complexity $N = \text{poly}(d, \epsilon^{-1}, \log(\delta^{-1}), n, t)$, by Lemma~\ref{lem:utm-expressibility}, we have with probability $1-\delta$ over the randomness in the data that the probability  that the $c(t)$ step of the computation history of $\text{MHLA}_{\hat{\Theta}}$ is equal to $x_t$ is
 \begin{equation}
\Pr\nolimits_{(M,x) \sim \mathcal{D}} \left[\text{CH}_{\hat{\Theta}}(M,x)^{c(t)}[:-k_t] = x^t\right] \geq 1 - \epsilon,
 \end{equation}
where $c(t) \leq O((n+t)t)$.  That is, the computation history of the $\text{MHLA}$ returned by \algref{alg:poly} is equal to the computation history of $M$ on $x$.   
\end{corollary}

\begin{proof}
We have from \thmref{thm:learning} that \algref{alg:poly} returns $\hat{\Theta}$ such that 
\begin{equation}
\E_{(Z,y) \in \mathcal{D}}\left[\left(\text{MHLA}_{\hat{\Theta}}(Z) - y\right)^2\right] - \min_{\Theta \in \Omega_H}\E_{(Z,y) \in \mathcal{D}}\left[\left(\text{MHLA}_{\Theta}(Z) - y\right)^2\right]\leq \epsilon  
\end{equation}
Then to obtain an error bound on the $\Phi$ step computation history, which involves $O(n\Phi)$ tokens, we just observe that by union bound each step rounds to an incorrect set of tokens with probability less than $\epsilon$.  Therefore, over $O(\Phi)$ steps the error probability is upper bounded by $\epsilon \Phi$.  Equivalently 
\begin{equation}
\Pr_{(M,x) \sim \mathcal{D}}[\text{CH}_{\hat{\Theta}}(M,x) \neq \text{CH}_{\Theta}(M,x)] \leq O(\epsilon \Phi).
\end{equation}

Then proving \corref{cor:learning-utm} is a simple exercise.  For a larger sample complexity $N = \text{poly}(d, \epsilon^{-1}, \log(\delta^{-1}), n, t)$, by Lemma~\ref{lem:utm-expressibility}, we have that the probability that every token of the autoregressive computation history of $\text{MHLA}_{\hat{\Theta}}$ is equal to $x_t$ is 
 \begin{equation}
\Pr_{(M,x) \sim \mathcal{D}} \left[\text{CH}_{\hat{\Theta}}(M,x)^{c(t)}[:-k_t] = x^t\right] \geq 1 - \epsilon
\end{equation}

\end{proof}

\learnUTMcert*

\begin{proof} \label{proof:learnUTMcert}
The proof follows from the quantitative version of \lemref{lem:error-identifiability}.  Using the given that $\lambda_{min}(\Lambda_D) > \eta$, we conclude that  
 for any $\hat{\Theta} \in \Omega_{\epsilon-\text{ERM}}$ that for all inputs $Z \in \R^{d \times n}$ 
\begin{equation}
\|\text{MHLA}_{\hat{\Theta}}(Z) - \text{MHLA}_{\Theta}(Z)\| \leq   \frac{\epsilon}{\eta} \|Z\|_F^6.
\end{equation}

If we select a sufficiently small $\epsilon = 1/\text{poly}(d,N,|Q|,|\Sigma|,n,t,\eta^{-1})$ then we can ensure 
 \begin{equation}
\Pr_{(M,x) \sim \mathcal{D}} \left[\text{CH}_{\hat{\Theta}}(M,x)^{c(t)}[:-k_t] = x^t\right] \geq 1 - \epsilon
\end{equation}.  

The runtime then scales with $\text{poly}(d,N,|Q|,|\Sigma|,n,t,\eta^{-1})$ as desired.  
\end{proof}

\clearpage
\section{Additional Definitions}
\label{sec:add-def}

\begin{definition} [Orthogonal Embeddings] Let $\text{Embed}$ be a function $\text{Embed}: \Sigma \rightarrow \R^{|\Sigma|}$.  Let $\Sigma$ be an alphabet and let $e_1,e_2,...,e_{|\Sigma|} \in \R^{|\Sigma|}$ be a basis of orthogonal unit vectors.  Then for each letter $a$ in an alphabet $\Sigma$, we define $\text{Embed}(a) = e_a$ where we associate a different unit vector to each letter. 
\end{definition}

We adopt a naive "rounding" scheme for converting vectors into tokens.  This can be done in a variety of ways, and we choose to simply round the entries of the vector embeddings to the nearest token embedding.  
\begin{definition} [Rounding] For any vector $v = (v_1,v_2,...,v_d) \in \R^d$, let $\text{Round}(v) = e_j$ for $j = \argmax_{i \in [d]} \langle v, e_i\rangle$.  Since we use orthogonal unit vectors for token embeddings we will refer to $\text{Round}(v)$ as a token.  We will often refer to a matrix $Z \in \R^{d \times n}$ as being equivalent to a series of $n$ tokens $a_1,a_2,...,a_n$ to mean $\text{Round}(Z[:,i]) = a_i$ for all $i \in [n]$.  
\end{definition}

\begin{algorithm}
\small
\caption{Extract Features}
\label{alg:extract}
\begin{algorithmic}[1]
\STATE{\textbf{Input: } Data $D \defeq \{Z_i\}_{i \in [N]}$ for $Z_i \in \R^{d \times n_i}$ and $y_i \in \R^d$}
\FOR{$Z_i \in D$}
\STATE{Let $z_1,z_2,...z_d$ be the rows of $Z_i$ and let $z_{a,b}$ be the $(a,b)$ entry of $Z_i$}
\FOR{$j \in  [d]$ }
	\FOR{$k \in [d]$}
		\FOR{$\ell \in [d]$}
			\STATE{Let $\mathcal{X}_i \in \R^{d \times d^2}$ be defined as follows}
    			\STATE{$\mathcal{X}_i [j,kd + \ell] =  \left[\langle z_{j:}, z_{k:} \rangle z_{\ell n_i}\right]$}
		\ENDFOR
	\ENDFOR
\ENDFOR
\ENDFOR
\STATE{\textbf{Return: } $\{\mathcal{X}_i\}_{i \in [N]}$ such that \begin{equation}\label{eq:features}
\mathcal{X}_i \defeq \begin{bmatrix} \langle z_{1:}, z_{1:} \rangle z_{1 n_i} & \langle z_{1:}, z_{2:} \rangle z_{1 n_i} & \cdots & \langle z_{1:}, z_{d:} \rangle z_{1 n_i}& \cdots  & \langle z_{1:}, z_{d:}\rangle z_{dn_i} \\
\langle z_{2:}, z_{1:} \rangle z_{1 n_i} & \langle z_{2:}, z_{2:} \rangle z_{1 n_i} & \cdots & \langle z_{2:}, z_{d:} \rangle z_{1 n_i} & \cdots & \langle z_{2:}, z_{d:}\rangle z_{dn_i}\\
\vdots & \vdots & \ddots & \vdots & \ddots & \vdots \\
\langle z_{d:}, z_{1:} \rangle z_{1 n_i} & \langle z_{d:}, z_{2:} \rangle z_{1 n_i} & \cdots &  \langle z_{d:}, z_{d:} \rangle z_{1 n_i} & \cdots & \langle z_{d:}, z_{d:}\rangle z_{dn_i}\\
\end{bmatrix} ~ .
\end{equation}}
\end{algorithmic}
\end{algorithm}

\subsection{Training details of attention networks}\label{app:train-details-la}
We use Adam \cite{kingma2014adam} optimizer to train linear attention model \cref{eq:LA} and the full Transformer \cite{vaswani2017attention} models. 
\begin{center}
\begin{tabular}{ll}
    \toprule
    \textbf{hyper parameter} & \textbf{search space} \\ \midrule
$d$ input dimension  & [2, 4, 8, 16] \\
$m$ number of heads  & [1, 2, 4, 8, 16] \\
$n$ number of layers  & [1, 2, 4] \\
learning rate  & [0.01, 0.001] \\
batch size  & [32, 64] \\
optimizer & AdamW \cite{loshchilov2018fixing} \\
\bottomrule
\end{tabular}
\end{center}

\subsection{Training details in DFA Execution}\label{app:train-details-dfa}

We use the Llama variant of the Transformer arhitecture from \citet{touvron2023llama}. We run each setting with $N$ number of training examples with the following different values $ N \in \{16, 32, 64, 128, 256, 512, 1024, 2048, 4096, 6144, 8192, 12290, 16384, 20480, 32768, 65536\}$. The other hyper parameters are given in the below table.

\begin{center}
\begin{tabular}{ll}
    \toprule
    \textbf{hyper parameter} & \textbf{search space} \\ \midrule
$d$ input dimension  & [2048] \\
$m$ number of heads  & [16] \\
$n$ number of layers  & [4] \\
learning rate  & [0.00025] \\
epochs & 100 \\
optimizer & AdamW \cite{loshchilov2018fixing}\\
\bottomrule
\end{tabular}
\end{center}

\subsection{Additional Experiments}

\begin{figure}[H]
    \centering
    \begin{subfigure}{0.5\textwidth}
        \centering
        \includegraphics[width=\textwidth]{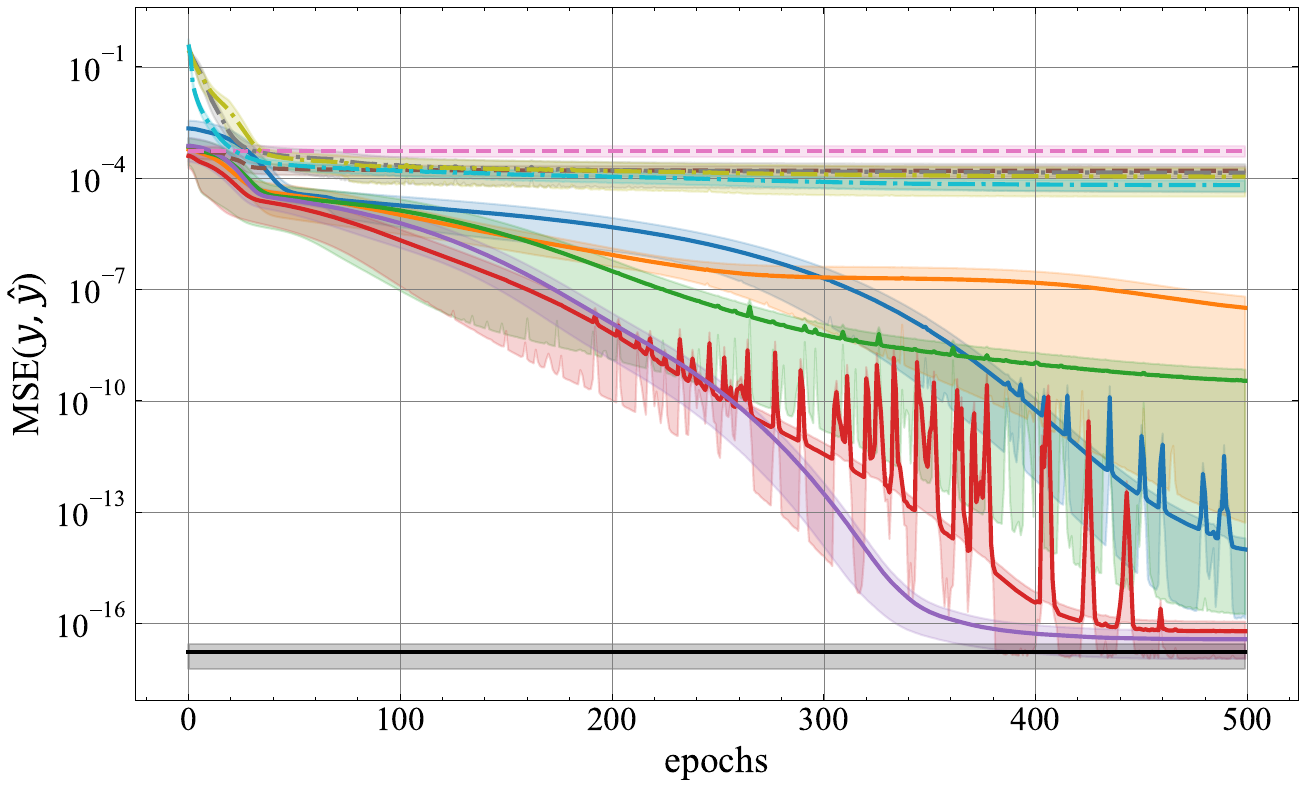}
        \caption{$N=2048, d=4$}
        \label{fig:val-loss-n-2048-d-4}
    \end{subfigure}\hfill
    \caption{\textbf{Performance comparison of multi-head, multi-layer linear attention models and the original Transformer model (denoted as \emph{full})}. We trained using SGD on synthetic data generated from a single-layer linear attention model for varying training set sizes ($N$) and input dimensions ($d$), number of heads $m$, and number of layers $n$. We present mean squared error of the predictions w.r.t number of training epochs. Results demonstrate that multi-head architectures converge faster on different input dimensions and match the performance of our \algref{alg:poly} (convex algorithm). Increasing the number of layers or incorporating multilayer perceptrons (MLPs) and layer normalization did not yield consistent improvements. Shading indicates the standard error over three different runs.}
    \label{fig:val-loss-la}
\end{figure}

\subsection{Learning the Computation History of Deterministic Finite Automata}\label{sec:utmdfa}

Universal automata (like the universal Turing machine discussed in \cref{sec:learn-UTM}) receive descriptions of other automata as input, and simulate them to produce an output.
Here we empirically evaluate the ability of MHLA models to perform universal simulation of deterministic finite automata (DFAs). We limit our study to DFAs with a maximum number of states ($N$), alphabet size ($V$), and input length  ($L$). While recent work on in-context learning \citep{akyürek2024} has focused on inferring DFA behavior from input--output examples, here, we aim to simulate DFAs given explicit descriptions of their state transitions as input---a task somewhat analogous to \emph{instruction following} in large scale language models.

The construction in \lemref{lem:learning-utm} shows that a linear attention layer can output the polynomially bounded computation history of any TM (and therefore any DFA). Our construction requires embedding size linear with maximum length of computation history, number of states and alphabet size. Therefore, we predict the data requirements are polynomial in each of $N, V$ and $L$.

\begin{figure}[!htbp]
    \centering
    \includegraphics[width=0.99\linewidth]{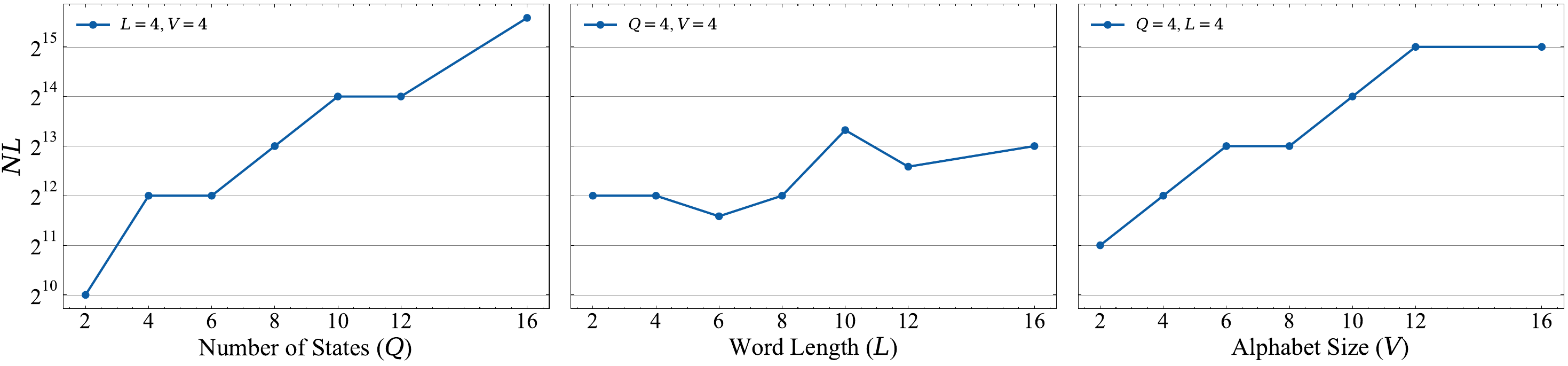}
    \caption{\textbf{Data requirement for universal DFA simulation:} We train a fixed sized Transformer (4-layers, 16 heads and 2048 hidden dimensions) to simulate a DFA given a transition table and input word. The vertical axis shows the number of tokens (expressed as word length $L$ times the number of examples $Q$) required to obtain 99\% next token accuracy.}
    \label{fig:dfa}
\end{figure}

\paragraph{Dataset}
Our dataset consists of strings containing three components: the input DFA's transition function $\delta: \mathcal{Q} \times \Sigma \rightarrow \mathcal{Q}$, the input word $x \in \Sigma^L$ and the computation history $h \in \mathcal{Q}^L$ which is the sequence of states visited in the DFA as it decides if $x$ is in its language. The first two components are the input to the model, while the computation history is the target output. We adopt the following schema for representing $\delta,x,$ and $h$:
$$\underbrace{(s_i, w, s_j), \dots, \forall_{s_i \in \mathcal{Q},w \in \Sigma} \in \delta}_{\text{DFA transition function}} \mid \underbrace{w_0 w_1 \dots w_L}_{\text{word}} \mid \underbrace{(s^0 w_0 s^1), (s^1 w_1 s^2) , \dots , (s^{L-1} w_L s^L)}_{\text{computation history}}$$
We encode each input-output relation in the transition function as a sequence of three tokens $(s_i, w, s_j)$ where $\delta(s_i,w) = s_j$.  We also include two parantheses to separate each triplet of tokens for a total of five tokens for each input-output relation.  The total description length of $\delta$ is then $5\mathcal{Q}\Sigma$. We encode word $x$ of length $L$ as a sequence of $L$ tokens.  Finally, we encode the computation history as the sequence of state transitions the DFA visits when deciding if $x$ is in its language.  Here we designate $s0$ as the start state, and let $s^i = \delta(s^{i-1},w^{i-1})$.  Each state transition is again represented by a triplet $(s, w, \delta(s,w))$.    
We train an autoregressive Transformer model using cross-entropy loss to predict the computation history tokens given the transition function and word. Please refer to \cref{app:train-details-dfa} for hyperparameter details.
\paragraph{Results}
In \cref{fig:dfa}, we vary each of the parameters $Q$, $L$ and $V$, while the other two parameters are fixed to a constant (in this case we fix them to be 4). Then, on the vertical axis, we display the minimum number of tokens (number of examples times the word length) required to get 99\% accuracy on the next token prediction. Plots are suggestive of a sub-exponential dependence on DFA complexity.

%% file: neurips_2025.bbl
\begin{thebibliography}{47}
\providecommand{\natexlab}[1]{#1}
\providecommand{\url}[1]{\texttt{#1}}
\expandafter\ifx\csname urlstyle\endcsname\relax
  \providecommand{\doi}[1]{doi: #1}\else
  \providecommand{\doi}{doi: \begingroup \urlstyle{rm}\Url}\fi

\bibitem[Ahn et~al.(2024)Ahn, Cheng, Song, Yun, Jadbabaie, and Sra]{ahn2024}
Kwangjun Ahn, Xiang Cheng, Minhak Song, Chulhee Yun, Ali Jadbabaie, and Suvrit Sra.
\newblock Linear attention is (maybe) all you need (to understand transformer optimization), 2024.
\newblock URL \url{https://arxiv.org/abs/2310.01082}.

\bibitem[Akyürek et~al.(2024)Akyürek, Wang, Kim, and Andreas]{akyürek2024}
Ekin Akyürek, Bailin Wang, Yoon Kim, and Jacob Andreas.
\newblock In-context language learning: Architectures and algorithms, 2024.
\newblock URL \url{https://arxiv.org/abs/2401.12973}.

\bibitem[Barcel{\'o} et~al.(2020)Barcel{\'o}, Kostylev, Monet, P{\'e}rez, Reutter, and Silva]{barcelo2020logical}
Pablo Barcel{\'o}, Egor~V Kostylev, Mikael Monet, Jorge P{\'e}rez, Juan Reutter, and Juan-Pablo Silva.
\newblock The logical expressiveness of graph neural networks.
\newblock In \emph{ICLR}, 2020.

\bibitem[Barcel{\'o} et~al.(2024)Barcel{\'o}, Kozachinskiy, Lin, and Podolskii]{barcelo2023logical}
Pablo Barcel{\'o}, Alexander Kozachinskiy, Anthony~Widjaja Lin, and Vladimir Podolskii.
\newblock Logical languages accepted by transformer encoders with hard attention.
\newblock 2024.

\bibitem[Beck et~al.(2024)Beck, P{\"o}ppel, Spanring, Auer, Prudnikova, Kopp, Klambauer, Brandstetter, and Hochreiter]{beck2024xlstm}
Maximilian Beck, Korbinian P{\"o}ppel, Markus Spanring, Andreas Auer, Oleksandra Prudnikova, Michael Kopp, G{\"u}nter Klambauer, Johannes Brandstetter, and Sepp Hochreiter.
\newblock x{LSTM}: Extended long short-term memory.
\newblock Vancouver, Canada, December 2024.

\bibitem[Bietti et~al.(2023)Bietti, Cabannes, Bouchacourt, Jegou, and Bottou]{bietti2023}
Alberto Bietti, Vivien Cabannes, Diane Bouchacourt, Herve Jegou, and Leon Bottou.
\newblock Birth of a transformer: A memory viewpoint, 2023.
\newblock URL \url{https://arxiv.org/abs/2306.00802}.

\bibitem[Cabannes et~al.(2024)Cabannes, Simsek, and Bietti]{cabannes2024}
Vivien Cabannes, Berfin Simsek, and Alberto Bietti.
\newblock Learning associative memories with gradient descent, 2024.
\newblock URL \url{https://arxiv.org/abs/2402.18724}.

\bibitem[Chen and Li(2024)]{chen2024}
Sitan Chen and Yuanzhi Li.
\newblock Provably learning a multi-head attention layer, 2024.
\newblock URL \url{https://arxiv.org/abs/2402.04084}.

\bibitem[Chiang et~al.(2023)Chiang, Cholak, and Pillay]{chiang2023tighter}
David Chiang, Peter Cholak, and Anand Pillay.
\newblock Tighter bounds on the expressivity of transformer encoders.
\newblock \emph{arXiv preprint arXiv:2301.10743}, 2023.

\bibitem[Dao and Gu(2024)]{dao2024transformersssmsgeneralizedmodels}
Tri Dao and Albert Gu.
\newblock Transformers are ssms: Generalized models and efficient algorithms through structured state space duality, 2024.
\newblock URL \url{https://arxiv.org/abs/2405.21060}.

\bibitem[Dehghani et~al.(2018)Dehghani, Gouws, Vinyals, Uszkoreit, and Kaiser]{dehghani2018universal}
Mostafa Dehghani, Stephan Gouws, Oriol Vinyals, Jakob Uszkoreit, and {\L}ukasz Kaiser.
\newblock Universal transformers.
\newblock \emph{arXiv preprint arXiv:1807.03819}, 2018.

\bibitem[Deora et~al.(2023)Deora, Ghaderi, Taheri, and Thrampoulidis]{deora2023}
Puneesh Deora, Rouzbeh Ghaderi, Hossein Taheri, and Christos Thrampoulidis.
\newblock On the optimization and generalization of multi-head attention, 2023.
\newblock URL \url{https://arxiv.org/abs/2310.12680}.

\bibitem[Dong et~al.(2019)Dong, Mao, Lin, Wang, Li, and Zhou]{dong2019neural}
Honghua Dong, Jiayuan Mao, Tian Lin, Chong Wang, Lihong Li, and Denny Zhou.
\newblock Neural logic machines.
\newblock In \emph{ICLR}, 2019.

\bibitem[Edelman et~al.(2022{\natexlab{a}})Edelman, Goel, Kakade, and Zhang]{edelman2022}
Benjamin~L. Edelman, Surbhi Goel, Sham Kakade, and Cyril Zhang.
\newblock Inductive biases and variable creation in self-attention mechanisms, 2022{\natexlab{a}}.
\newblock URL \url{https://arxiv.org/abs/2110.10090}.

\bibitem[Edelman et~al.(2022{\natexlab{b}})Edelman, Goel, Kakade, and Zhang]{edelman2022inductive}
Benjamin~L Edelman, Surbhi Goel, Sham Kakade, and Cyril Zhang.
\newblock Inductive biases and variable creation in self-attention mechanisms.
\newblock In \emph{International Conference on Machine Learning}, pages 5793--5831. PMLR, 2022{\natexlab{b}}.

\bibitem[Feng et~al.(2023)Feng, Gu, Zhang, Ye, He, and Wang]{feng2023towards}
Guhao Feng, Yuntian Gu, Bohang Zhang, Haotian Ye, Di~He, and Liwei Wang.
\newblock Towards revealing the mystery behind chain of thought: a theoretical perspective.
\newblock \emph{arXiv preprint arXiv:2305.15408}, 2023.

\bibitem[Fu et~al.(2023)Fu, Guo, Bai, and Mei]{fu2023}
Hengyu Fu, Tianyu Guo, Yu~Bai, and Song Mei.
\newblock What can a single attention layer learn? a study through the random features lens, 2023.
\newblock URL \url{https://arxiv.org/abs/2307.11353}.

\bibitem[Giannou et~al.(2023)Giannou, Rajput, Sohn, Lee, Lee, and Papailiopoulos]{giannou2023looped}
Angeliki Giannou, Shashank Rajput, Jy-yong Sohn, Kangwook Lee, Jason~D Lee, and Dimitris Papailiopoulos.
\newblock Looped transformers as programmable computers.
\newblock \emph{arXiv preprint arXiv:2301.13196}, 2023.

\bibitem[Hahn(2020)]{hahn2020theoretical}
Michael Hahn.
\newblock Theoretical limitations of self-attention in neural sequence models.
\newblock \emph{Transactions of the Association for Computational Linguistics}, 8:\penalty0 156--171, 2020.

\bibitem[Hao et~al.(2022)Hao, Angluin, and Frank]{hao2022formal}
Yiding Hao, Dana Angluin, and Robert Frank.
\newblock Formal language recognition by hard attention transformers: Perspectives from circuit complexity.
\newblock \emph{Transactions of the Association for Computational Linguistics}, 10:\penalty0 800--810, 2022.

\bibitem[Jelassi et~al.(2022)Jelassi, Sander, and Li]{jelassi2022}
Samy Jelassi, Michael~E. Sander, and Yuanzhi Li.
\newblock Vision transformers provably learn spatial structure, 2022.
\newblock URL \url{https://arxiv.org/abs/2210.09221}.

\bibitem[Katharopoulos et~al.(2020)Katharopoulos, Vyas, Pappas, and Fleuret]{katharopoulos2020}
Angelos Katharopoulos, Apoorv Vyas, Nikolaos Pappas, and François Fleuret.
\newblock Transformers are rnns: Fast autoregressive transformers with linear attention, 2020.
\newblock URL \url{https://arxiv.org/abs/2006.16236}.

\bibitem[Kearns and Vazirani(1994)]{kearns94}
Michael~J. Kearns and Umesh Vazirani.
\newblock \emph{{An Introduction to Computational Learning Theory}}.
\newblock The MIT Press, 08 1994.
\newblock ISBN 9780262276863.
\newblock \doi{10.7551/mitpress/3897.001.0001}.
\newblock URL \url{https://doi.org/10.7551/mitpress/3897.001.0001}.

\bibitem[Kingma and Ba(2014)]{kingma2014adam}
Diederik~P Kingma and Jimmy Ba.
\newblock Adam: A method for stochastic optimization.
\newblock \emph{arXiv preprint arXiv:1412.6980}, 2014.

\bibitem[Li et~al.(2024)Li, Liu, Zhou, and Ma]{li2024chain}
Zhiyuan Li, Hong Liu, Denny Zhou, and Tengyu Ma.
\newblock Chain of thought empowers transformers to solve inherently serial problems.
\newblock \emph{arXiv preprint arXiv:2402.12875}, 2024.

\bibitem[Liu et~al.(2022)Liu, Ash, Goel, Krishnamurthy, and Zhang]{liu2022transformers}
Bingbin Liu, Jordan~T Ash, Surbhi Goel, Akshay Krishnamurthy, and Cyril Zhang.
\newblock Transformers learn shortcuts to automata.
\newblock \emph{arXiv preprint arXiv:2210.10749}, 2022.

\bibitem[Loshchilov and Hutter(2018)]{loshchilov2018fixing}
Ilya Loshchilov and Frank Hutter.
\newblock Fixing weight decay regularization in adam, 2018.
\newblock URL \url{https://openreview.net/forum?id=rk6qdGgCZ}.

\bibitem[Merrill and Sabharwal(2023)]{merrill2023parallelism}
William Merrill and Ashish Sabharwal.
\newblock The parallelism tradeoff: Limitations of log-precision transformers.
\newblock \emph{Transactions of the Association for Computational Linguistics}, 11:\penalty0 531--545, 2023.

\bibitem[Merrill et~al.(2021)Merrill, Goldberg, and Smith]{merrill2021power}
William Merrill, Yoav Goldberg, and Noah~A Smith.
\newblock On the power of saturated transformers: A view from circuit complexity.
\newblock \emph{arXiv preprint arXiv:2106.16213}, 2021.

\bibitem[Merrill et~al.(2022)Merrill, Sabharwal, and Smith]{merrill2022saturated}
William Merrill, Ashish Sabharwal, and Noah~A Smith.
\newblock Saturated transformers are constant-depth threshold circuits.
\newblock \emph{Transactions of the Association for Computational Linguistics}, 10:\penalty0 843--856, 2022.

\bibitem[Oymak et~al.(2023)Oymak, Rawat, Soltanolkotabi, and Thrampoulidis]{oymak2023}
Samet Oymak, Ankit~Singh Rawat, Mahdi Soltanolkotabi, and Christos Thrampoulidis.
\newblock On the role of attention in prompt-tuning, 2023.
\newblock URL \url{https://arxiv.org/abs/2306.03435}.

\bibitem[P{\'e}rez et~al.(2019)P{\'e}rez, Marinkovi{\'c}, and Barcel{\'o}]{perez2019turing}
Jorge P{\'e}rez, Javier Marinkovi{\'c}, and Pablo Barcel{\'o}.
\newblock On the turing completeness of modern neural network architectures.
\newblock In \emph{ICLR}, 2019.

\bibitem[P{\'e}rez et~al.(2021)P{\'e}rez, Barcel{\'o}, and Marinkovic]{perez2021attention}
Jorge P{\'e}rez, Pablo Barcel{\'o}, and Javier Marinkovic.
\newblock Attention is turing complete.
\newblock \emph{The Journal of Machine Learning Research}, 22\penalty0 (1):\penalty0 3463--3497, 2021.

\bibitem[Schlag et~al.(2021)Schlag, Irie, and Schmidhuber]{schlag2021linear}
Imanol Schlag, Kazuki Irie, and J\"urgen Schmidhuber.
\newblock Linear {T}ransformers are secretly fast weight programmers.
\newblock Virtual only, July 2021.

\bibitem[Strobl et~al.(2024)Strobl, Merrill, Weiss, Chiang, and Angluin]{strobl2024formal}
Lena Strobl, William Merrill, Gail Weiss, David Chiang, and Dana Angluin.
\newblock What formal languages can transformers express? a survey.
\newblock \emph{Transactions of the Association for Computational Linguistics}, 12:\penalty0 543--561, 2024.

\bibitem[Sun et~al.(2023)Sun, Dong, Huang, Ma, Xia, Xue, Wang, and Wei]{sun2023retentive}
Yutao Sun, Li~Dong, Shaohan Huang, Shuming Ma, Yuqing Xia, Jilong Xue, Jianyong Wang, and Furu Wei.
\newblock Retentive network: A successor to transformer for large language models.
\newblock \emph{Preprint arXiv:2307.08621}, 2023.

\bibitem[Tarzanagh et~al.(2024)Tarzanagh, Li, Thrampoulidis, and Oymak]{tarzanagh2024}
Davoud~Ataee Tarzanagh, Yingcong Li, Christos Thrampoulidis, and Samet Oymak.
\newblock Transformers as support vector machines, 2024.
\newblock URL \url{https://arxiv.org/abs/2308.16898}.

\bibitem[Tian et~al.(2023)Tian, Wang, Chen, and Du]{tian2023}
Yuandong Tian, Yiping Wang, Beidi Chen, and Simon Du.
\newblock Scan and snap: Understanding training dynamics and token composition in 1-layer transformer, 2023.
\newblock URL \url{https://arxiv.org/abs/2305.16380}.

\bibitem[Touvron et~al.(2023)Touvron, Martin, Stone, Albert, Almahairi, Babaei, Bashlykov, Batra, Bhargava, Bhosale, et~al.]{touvron2023llama}
Hugo Touvron, Louis Martin, Kevin Stone, Peter Albert, Amjad Almahairi, Yasmine Babaei, Nikolay Bashlykov, Soumya Batra, Prajjwal Bhargava, Shruti Bhosale, et~al.
\newblock Llama 2: Open foundation and fine-tuned chat models.
\newblock \emph{arXiv preprint arXiv:2307.09288}, 2023.

\bibitem[Trauger and Tewari(2023)]{trauger2023}
Jacob Trauger and Ambuj Tewari.
\newblock Sequence length independent norm-based generalization bounds for transformers, 2023.
\newblock URL \url{https://arxiv.org/abs/2310.13088}.

\bibitem[Vaswani et~al.(2017)Vaswani, Shazeer, Parmar, Uszkoreit, Jones, Gomez, Kaiser, and Polosukhin]{vaswani2017attention}
Ashish Vaswani, Noam Shazeer, Niki Parmar, Jakob Uszkoreit, Llion Jones, Aidan~N Gomez, {\L}ukasz Kaiser, and Illia Polosukhin.
\newblock Attention is all you need.
\newblock \emph{Advances in neural information processing systems}, 30, 2017.

\bibitem[Wei et~al.(2021)Wei, Chen, and Ma]{wcm21}
Colin Wei, Yining Chen, and Tengyu Ma.
\newblock Statistically meaningful approximation: a case study on approximating turing machines with transformers.
\newblock \emph{CoRR}, abs/2107.13163, 2021.
\newblock URL \url{https://arxiv.org/abs/2107.13163}.

\bibitem[Weiss et~al.(2021)Weiss, Goldberg, and Yahav]{weiss2021thinking}
Gail Weiss, Yoav Goldberg, and Eran Yahav.
\newblock Thinking like transformers.
\newblock In \emph{International Conference on Machine Learning}, pages 11080--11090. PMLR, 2021.

\bibitem[Yang et~al.(2024)Yang, Wang, Shen, Panda, and Kim]{YangWSPK24}
Songlin Yang, Bailin Wang, Yikang Shen, Rameswar Panda, and Yoon Kim.
\newblock Gated linear attention transformers with hardware-efficient training.
\newblock Vienna, Austria, July 2024.

\bibitem[Yang et~al.(2025)Yang, Wang, Zhang, Shen, and Kim]{yang2025}
Songlin Yang, Bailin Wang, Yu~Zhang, Yikang Shen, and Yoon Kim.
\newblock Parallelizing linear transformers with the delta rule over sequence length, 2025.
\newblock URL \url{https://arxiv.org/abs/2406.06484}.

\bibitem[Zhang et~al.(2023)Zhang, Frei, and Bartlett]{zhang2023}
Ruiqi Zhang, Spencer Frei, and Peter~L. Bartlett.
\newblock Trained transformers learn linear models in-context, 2023.
\newblock URL \url{https://arxiv.org/abs/2306.09927}.

\bibitem[Zhang et~al.(2024)Zhang, Liu, Cai, Wang, and Wang]{zhang2024}
Yufeng Zhang, Boyi Liu, Qi~Cai, Lingxiao Wang, and Zhaoran Wang.
\newblock An analysis of attention via the lens of exchangeability and latent variable models, 2024.
\newblock URL \url{https://arxiv.org/abs/2212.14852}.

\end{thebibliography}
